\documentclass[journal]{IEEEtran}

\pdfoutput=1

\usepackage{amsmath}
\usepackage{amsthm}
\usepackage{amssymb}
\usepackage{longtable}
\usepackage{amsthm}
\usepackage{graphicx}
\usepackage{epstopdf}
\usepackage{subfigure}
\usepackage{mathrsfs}
\usepackage{eufrak}
\usepackage{enumerate}
\usepackage{url}
\usepackage{booktabs}
\usepackage{setspace}
\usepackage{multirow}
\usepackage{threeparttable}
\newtheorem{proposition}{Proposition}
\usepackage[noend]{algpseudocode} 
\usepackage{algorithmicx,algorithm}
\usepackage[symbol]{footmisc}

\usepackage{hyperref} 
\usepackage[numbers]{natbib}
\begin{document}

\title{Towards Consistency and Complementarity: A Multiview Graph Information Bottleneck Approach}
\author{Xiaolong~Fan,
        Maoguo~Gong,~\IEEEmembership{Senior Member,~IEEE}, 
        Yue~Wu,~\IEEEmembership{Member,~IEEE},\\
        Mingyang~Zhang~\IEEEmembership{Member,~IEEE},
        Hao~Li,
        and Xiangming~Jiang

\thanks{X. Fan, M. Gong (corresponding author), M. Zhang, H. Li, and X. Jiang are with the School of Electronic Engineering, Key Laboratory of Intelligent Perception and Image Understanding, Ministry of Education, Xidian University, Xi'an, Shaanxi province, China, 710071. (e-mail: xiaolongfan@outlook.com; gong@ieee.org; haoli@xidian.edu.cn; myzhang@xidian.edu.cn; xmjiang@xidian.edu.cn)}
\thanks{Y. Wu is with the School of Computer Science and Technology, Xidian University, Xi'an, Shaanxi province, China, 710071. (e-mail: ywu@xidian.edu.cn)}}

\markboth{Preview}%
{}

\maketitle

\begin{abstract}
	The empirical studies of Graph Neural Networks (GNNs) broadly take the original node feature and adjacency relationship as singleview input, ignoring the rich information of multiple graph views. To circumvent this issue, the multiview graph analysis framework has been developed to fuse graph information across views. How to model and integrate shared (i.e. consistency) and view-specific (i.e. complementarity) information is a key issue in multiview graph analysis. In this paper, we propose a novel \underline{M}ultiview \underline{V}ariational \underline{G}raph \underline{I}nformation \underline{B}ottleneck (MVGIB) principle to maximize the agreement for common representations and the disagreement for view-specific representations. Under this principle, we formulate the common and view-specific information bottleneck objectives across multiviews by using constraints from mutual information. However, these objectives are hard to directly optimize since the mutual information is computationally intractable. To tackle this challenge, we derive variational lower and upper bounds of mutual information terms, and then instead optimize variational bounds to find the approximate solutions for the information objectives. Extensive experiments on graph benchmark datasets demonstrate the superior effectiveness of the proposed method.
\end{abstract}


\IEEEpeerreviewmaketitle

\section{Introduction}
\IEEEPARstart{A}{s} a ubiquitous data structure, graph is capable of modeling real-world systems in numerous domains, ranging from social network analysis \cite{liu2022nowhere,fan2019graph}, natural language processing \cite{zhao2021weighted,yao2019graph}, computer vision \cite{zhou2021event,zhang2019latentgnn}, and other domains. In recent years, as a powerful tool for learning and analyzing graph data, Graph Neural Networks (GNNs) \cite{scarselli2008graph,vgae,mpnn,gcn,gin} have received tremendous research attention and have been widely employed for graph analysis tasks. The common graph neural networks broadly follow the message passing framework \cite{mpnn} which involves a message passing phase and a readout phase. The message passing first iteratively aggregates the neighbor representations of each node to generate new node representations, and then the readout phase capture the global graph information from node representation space to generate the graph representation. The success of GNNs can be attributed to their ability to simultaneously exploit the rich information inherent in the singleview graph topology structure and input node attributes. 

However, for real-world applications, graph data are often manifested as multiple types of sources or different topology subsets, which can be naturally organized as multiview graph data. For example, a multiplex network \cite{park2020unsupervised,hdmi} contains multiple systems of the same set of nodes, and there exists various types of relationships among nodes, which can be seen as different topology subsets. Furthermore, existing approaches may confront challenges, such as difficulty in obtaining labeled data and generalization bias under the supervised learning setting \cite{liu2021self}. One way to alleviate these issues is unsupervised representation learning on graph. \textit{Therefore, how to integrate different graph views into low-dimensional representations for downstream tasks in an unsupervised manner becomes a fundamental problem for graph representation learning}.

Recently, several multiview graph representation learning approaches have been proposed to effectively explore the multiview graph data. For instance, Adaptive Multi-channel Graph Convolutional Networks (AM-GCN) \cite{amgcn} show that performing graph convolutional operation over both topology and feature views can improve the performance of node representation learning, where the feature view can be generated by distance based K-Nearest Neighbor algorithm. Note that AM-GCN works in the supervised learning manner and therefore may suffer from the challenge of annotating graphs. Contrastive Multiview Graph Representation Learning (MVGRL) \cite{mvgrl} presents an unsupervised mutual information maximization approach for learning node and graph representations by contrasting graph views, where the additional graph view is generated by structural augmentations such as Personalized PageRank diffusion and heat diffusion strategies. Specifically, MVGRL maximizes the mutual information between two views by contrasting node representations from one view with graph representation from the other view, but neglects to explicitly distinguish the common and view-specific information across multiviews. Recent work \cite{wan2021multi} has shown that explicitly modeling common and view-specific information can improve the performance of multiview representation model.

In this paper, we focus on the in-depth analysis of multiview graph representation learning and aim to answer the question that how to extract common and view-specific information in an unsupervised manner for graph representation learning. Inspired by the idea of mutual information constraints in information bottlenecks \cite{oib,dvib,gib,sun2022graph}, we first formulate the information bottleneck objective to encourage that the latent representation contains as much information as possible about the corresponding input view and as less information as possible about the other view to explore the complementary information across multiple graph views. Then, to further explore the complex associations among different graph views, we also formulate an information bottleneck objective to retain as much information as possible about the corresponding input view and the other view simultaneously. Note that, different from previous information bottleneck methods \cite{dvib,gib,sib} that discard superfluous information for the given task, the proposed method focuses on using mutual information to constrain the latent graph representations to model consistency and complementarity across graph multiviews. 

However, the mutual information terms are challenging to calculate in practice, since they require computing the intractable posterior distribution. To tackle this challenge, we follow the variational inference methods \cite{dvib,poole2019variational} and derive the variational lower bounds of mutual information terms that need to be maximized. For the mutual information terms that need to be minimized, we adopt Contrastive Log-ratio Upper Bound (CLUB) \cite{club} as the upper bound approximation. To enable the computation of CLUB, a key challenge is to model the intractable conditional distribution. Here, we employ the classical graph neural network encoder to variationally approximate conditional distribution. By combining the derived lower and upper bounds, we can easily optimize the overall objective by using the stochastic gradient descent algorithm. 

Extensive validation experiments on seven graph benchmark datasets demonstrate the superior effectiveness of the proposed method. Specifically, we evaluate the proposed method on the graph classification and graph clustering tasks, and also perform the ablation study to investigate the performance of certain components in the proposed method to understand the contribution to the overall system. 

To summarize, we outline the main contributions in this paper as below:   
\begin{itemize}
	\item[1.] We propose a novel Multiview Variational Graph Information Bottleneck (MVGIB) principle to maximize the agreement for common representations and the disagreement for view-specific representations.  
	\item[2.] We present the variational lower and upper bounds of the mutual information terms in MVGIB by exploiting the variational approximation method.
	\item[3.] Extensive validation experiments and ablation studies on seven graph benchmark datasets demonstrate the superior effectiveness of the proposed method.
\end{itemize}

The rest of this paper is organized as follows. In section \ref{s2}, we make a brief review of the related works of graph neural networks and information bottleneck. Section \ref{s3} describes the preliminaries. Section \ref{s4} gives the proposed method by introducing representation model towards complementarity and consistency, and then discusses the optimize and inference. Section \ref{s5} evaluates the proposed method with extensive experiments on graph representation benchmark datasets. Finally, section \ref{s6} concludes this paper and discusses the future work.          

\section{Related Work} \label{s2}
\subsection{Graph Neural Network}
Graph Neural Networks (GNNs) aim to model the non-Euclidean graph data structure and have been demonstrated to achieve state-of-the-art performance on graph analysis tasks. As a unified framework for graph neural networks, graph message passing neural network \cite{mpnn,propagation} generalizes several existing representative graph neural networks, such as Graph Convolutional Networks (GCN) \cite{gcn}, Graph Isomorphism Network (GIN) \cite{gin}, Deep Hierarchical Layer Aggregation (DHLA) \cite{fan2021deep}, and others. In recent years, several label-free contrastive self-supervised graph representation learning methods have been developed, such as Deep Graph Infomax (DGI) \cite{dgi}, Contrastive Multiview Graph Representation Learning (MVGRL) \cite{mvgrl}, InfoGraph \cite{infograph}, and have achieved competitive performance, even exceeding supervised graph representation learning. In this paper, we use the graph isomorphism network as graph encoder to generate the graph representation.

\subsection{Information Bottleneck} 
The information bottleneck methods for deep neural networks \cite{dvib} are developed to generate a compact but informative representation. Under this setting, if latent representation discards information from the input which is not useful for a given task, it can increase robustness for the downstream tasks. In recent years, several information bottleneck based graph representation learning methods have been proposed, such as Graph Information Bottleneck (GIB) \cite{gib}, Subgraph Information Bottleneck (SIB) \cite{sib}, and Variational Information Bottleneck for Graph Structure Learning (VIB-GSL) \cite{sun2022graph}. In this paper, we follow the idea of mutual information constraints in information bottlenecks and develop the multiview graph information bottleneck for graph representation learning.

\begin{figure*}[]
	\centering
	\includegraphics[width=1.97\columnwidth]{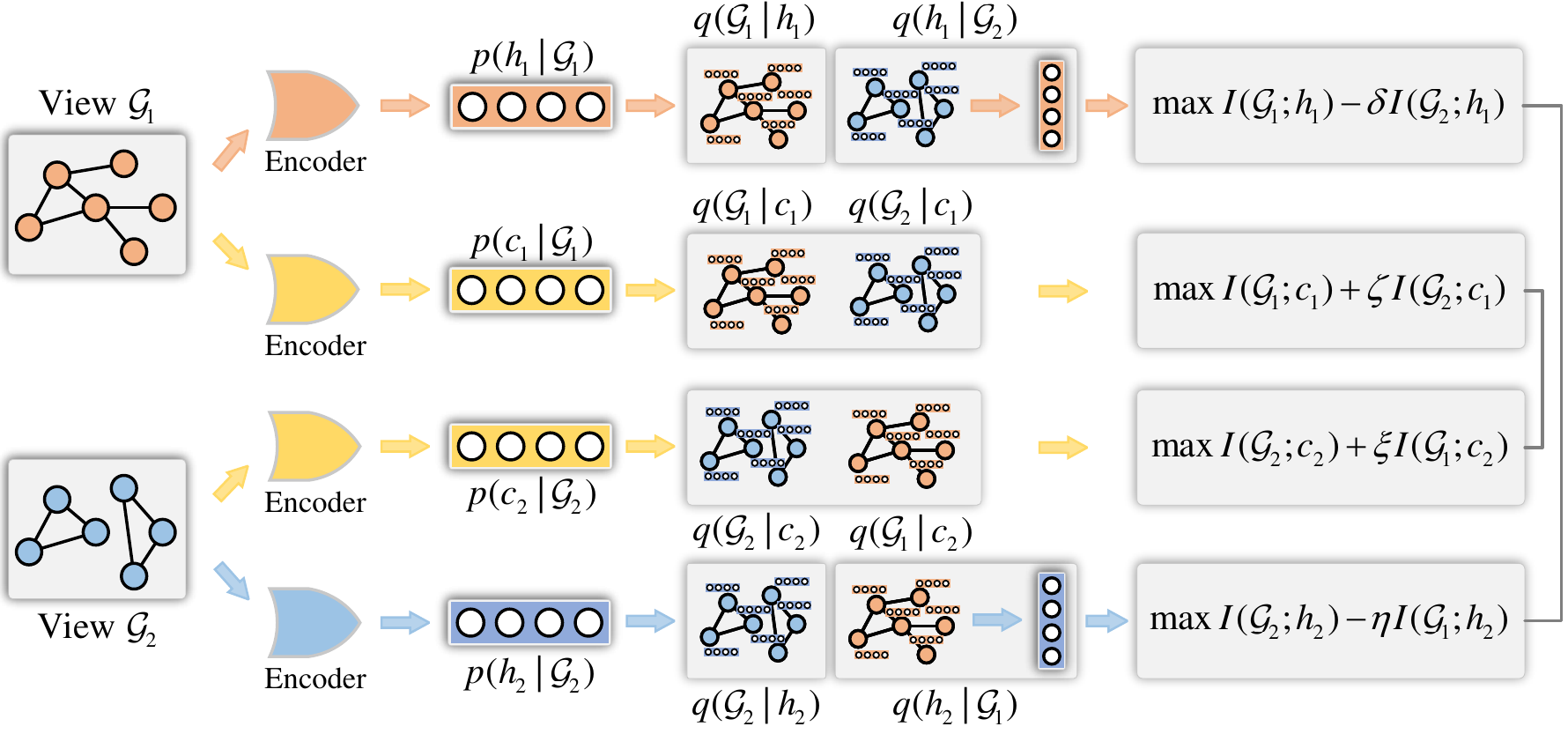}
	\caption{The overall framework of the proposed approach. Here, $\mathcal{G}_1$ with $(A_1, X)$ and $\mathcal{G}_2$ with $(A_2, X)$ to be two input graph views. To extract consistent relationships between multiview graphs, the proposed approach maximizes the mutual information between input graph and the common representation, while maximizing the mutual information between $\mathcal{G}_2$ and the common representation $c_1$. To extract complementarity relationships between multiview graphs, the proposed approach maximizes the mutual information between $\mathcal{G}_1$ and $h_1$, while minimizing the mutual information between $\mathcal{G}_2$ and $h_1$.}
	\label{framework}
\end{figure*} 

\section{Preliminaries} \label{s3}
This section first introduces notations and then discusses the recent variational information bottleneck method.

\subsection{Notations}
We first introduce the notations throughout this paper. Let $\mathcal{G} = (\mathcal{V}, \mathcal{E})$ be a graph with $\mathcal{V}$ and $\mathcal{E}$ denoting the node set and edge set, respectively. The nodes are described by the adjacency matrix $A \in \mathbb{R}^{n\times n}$ with the number of nodes $n$ where $A_{ij} = 1$ if the edge between $i$-th node and $j$-th node exists and else $A_{ij} = 0$, and the node feature matrix $X \in \mathbb{R}^{n\times d}$ with the number of feature $d$ per node. Under the multiview graph input setting, suppose the number of graph views is $r$, the graph contains a set of adjacency matrices, i.e., $\mathcal{A} = \{A_1, A_2, ..., A_r\}$, and a set of node features, i.e., $\mathcal{X} = \{X_1, X_2, ..., X_r\}$. 
Following the frequent notations in information theory, we use $H(x) = -\mathbb{E}_{p(x)}\log p(x)$ denoting Shannon entropy, $H(p,q) = -\mathbb{E}_{q(x)}\log p(x)$ denoting cross entropy, $I(x;y) = \mathbb{E}_{p(x,y)}\log \frac{p(x,y)}{p(x)p(y)}$ denoting mutual information, and $\mathcal{D}_{KL}(p\Vert q) = \mathbb{E}_{p(x)}\log \frac{p(x)}{q(x)}$ denoting Kullback-Leiber divergence.

\subsection{Information Bottleneck}
The Information Bottleneck (IB) method \cite{oib} is introduced as an information theoretic principle for extracting compact but informative representations. Given the input data $x$, representation $z$, and predict label $y$, the IB method maximizes the mutual information between representation $z$ and predicted label $y$, while constraining the mutual information between representation $z$ and input data $x$, i.e.,   
\begin{gather}
	\max I(z;y) \,\,\,\, \text{s.t.}\ I(x;z) \leq I_c \label{e1}
\end{gather} 
where $I_c$ denotes the information constraint. By introducing Lagrange multiplier $\alpha$, Equation \ref{e1} can be transformed as 
\begin{gather}
	\max I(z;y) - \alpha I(x;z). \label{e2}
\end{gather}
Intuitively, the first term of Equation \ref{e2} encourages $z$ to be predictive of $y$ and the second term encourages $z$ to forget $x$, resulting in the minimal but sufficient encoding $z$ for predicting $y$. However, due to the computational challenge of mutual information, the IB objective is hard to directly optimized. To compensate for this, recent work \cite{dvib} proposes a feasible way to approximate IB objective by incorporating variational inference in form of 
\begin{gather}
	\mathbb{E}_{p(x,y)} \mathbb{E}_{p(z\vert x)} \log q(y\vert z)
	-\alpha \mathbb{E}_{p(x)} \mathbb{E}_{p(z\vert x)} \log \frac{p(z\vert x)}{r(z)} 
\end{gather}
where $q(y\vert z)$ and $r(z)$ are the variational approximations to $p(y\vert z)$ and marginal distribution $p(z)$, respectively. 

\section{Proposed Approach} \label{s4}
As a motivating example, suppose $\mathcal{G}_1$ with $(A_1, X_1)$ and $\mathcal{G}_2$ with $(A_2, X_2)$ to be two input graph views. Without loss of generality, we assume that the different graph views differ only in their adjacency relationship, i.e., $X_1 = X_2$. To extract consistent relationship and complementary information between multiview graph data, we propose Multiview Variational Graph Information Bottleneck (MVGIB) by extending the information bottleneck principle to the unsupervised multiview graph representation setting. The overall framework of MVGIB is shown in Figure \ref{framework}.     

\subsection{Towards Complementarity}
Given two input graph views $\mathcal{G}_1$ and $\mathcal{G}_2$ of graph $\mathcal{G}$, complementarity principle aims to generate the view-specific representation $h_1$ and $h_2$ for the given graph views $\mathcal{G}_1$ and $\mathcal{G}_2$, respectively. Taking $f_{enc}: \mathcal{G}_1 \to h_1$ as an example, we expect to maximize the mutual information between $\mathcal{G}_1$ and $h_1$, while minimizing the mutual information between $\mathcal{G}_2$ and $h_1$, i.e., 
\begin{gather} \label{e26}
	\max I(\mathcal{G}_1;h_1) - \delta I(\mathcal{G}_2; h_1)
\end{gather} 
where $\delta$ is the trade-off parameter to balance $I(\mathcal{G}_1;h_1)$ and $I(\mathcal{G}_2; h_1)$. Analogously, we can formulate the information bottleneck objective for $f_{enc}: \mathcal{G}_2 \to h_2$, i.e.,  
\begin{gather} \label{e27}
	\max I(\mathcal{G}_2;h_2) - \eta I(\mathcal{G}_1; h_2)
\end{gather} 
where $\eta$ is the trade-off parameter. Intuitively, maximizing Equation \ref{e26} and Equation \ref{e27} encourage representation $h_1$ and $h_2$ to maximally express the input graph views generated from, while maximally compressing the other graph views, i.e., towards complementarity. However, the mutual information is computationally intractable since it requires computing the intractable posterior distribution, resulting in these objectives being hard to optimize. To tackle this challenge, we introduce the variational inference approach \cite{dvib} to derive the variational lower bounds of $I(\mathcal{G}_1;h_1)$ and $I(\mathcal{G}_2;h_2)$ and the variational upper bounds of $I(\mathcal{G}_2;h_1)$ and $I(\mathcal{G}_1;h_2)$, as shown in Proposition \ref{pro1} and Proposition \ref{pro2}.

\begin{proposition}[Lower bounds of $I(\mathcal{G}_1;h_1)$ and $I(\mathcal{G}_2;h_2)$] \label{pro1}
	For mutual information $I(\mathcal{G}_1;h_1)$ and $I(\mathcal{G}_2;h_2)$, we have 
	\begin{gather}
		I(\mathcal{G}_1;h_1) \geq \mathbb{E}_{p(\mathcal{G}_1)}\mathbb{E}_{p(h_1\vert \mathcal{G}_1)}\log q(\mathcal{G}_1\vert h_1) \\
		I(\mathcal{G}_2;h_2) \geq \mathbb{E}_{p(\mathcal{G}_2)}\mathbb{E}_{p(h_2\vert \mathcal{G}_2)}\log q(\mathcal{G}_2\vert h_2)
	\end{gather}
	where $q(\mathcal{G}_1\vert h_1)$ and $q(\mathcal{G}_2\vert h_2)$ are the variational approximations of $p(\mathcal{G}_1\vert h_1)$ and $p(\mathcal{G}_2\vert h_2)$, respectively.
\end{proposition}

\begin{proof}
	Let $q(\mathcal{G}_1\vert h_1)$ be a variational approximation to $p(\mathcal{G}_1\vert h_1)$. Using the fact that Kullback-Leiber divergence is always non-negative, we have
	\begin{gather}
		\mathcal{D}_{KL}\left(p(\mathcal{G}_1\vert h_1)\Vert q(\mathcal{G}_1\vert h_1)\right) \geq 0 \\
		\mathbb{E}_{p(\mathcal{G}_1\vert h_1)} \log \frac{p(\mathcal{G}_1\vert h_1)}{q(\mathcal{G}_1\vert h_1)} \geq 0 \\
		\mathbb{E}_{p(\mathcal{G}_1\vert h_1)} \log p(\mathcal{G}_1\vert h_1) \geq \mathbb{E}_{p(\mathcal{G}_1\vert h_1)} \log q(\mathcal{G}_1\vert h_1).
	\end{gather}
	Then, according to the definition of mutual information, we have
	\begin{equation}
		\begin{aligned}
			I(\mathcal{G}_1;h_1) &\geq \mathbb{E}_{p(\mathcal{G}_1,h_1)}\log \frac{q(\mathcal{G}_1\vert h_1)}{p(\mathcal{G}_1)}\\
			&= \mathbb{E}_{p(\mathcal{G}_1,h_1)}\log q(\mathcal{G}_1\vert h_1) - \mathbb{E}_{p(\mathcal{G}_1)}\log p(\mathcal{G}_1) \\
			&= \mathbb{E}_{p(\mathcal{G}_1,h_1)}\log q(\mathcal{G}_1\vert h_1) + H(\mathcal{G}_1).
		\end{aligned}
	\end{equation}
	Note that the entropy $H(\mathcal{G}_1)$ is independent of the objective optimization and can be ignored, hence we obtain
	\begin{equation}
		\begin{aligned}
			I(\mathcal{G}_1,h_1) &\geq \mathbb{E}_{p(\mathcal{G}_1,h_1)}\log q(\mathcal{G}_1\vert h_1) \\
			&= \mathbb{E}_{p(\mathcal{G}_1)}\mathbb{E}_{p(h_1\vert \mathcal{G}_1)}\log q(\mathcal{G}_1\vert h_1).
		\end{aligned}
	\end{equation}
	Similarly, we can derive the lower bound of $I(\mathcal{G}_2;h_2)$ as
	\begin{gather}
		I(\mathcal{G}_2;h_2) \geq \mathbb{E}_{p(\mathcal{G}_2)}\mathbb{E}_{p(h_2\vert \mathcal{G}_2)}\log q(\mathcal{G}_2\vert h_2).
	\end{gather}
\end{proof}

\begin{proposition}[Upper bounds of $I(\mathcal{G}_2;h_1)$ and $I(\mathcal{G}_1;h_2)$] \label{pro2}
	For mutual information $I(\mathcal{G}_2;h_1)$ and $I(\mathcal{G}_1;h_2)$, we have
	\begin{gather}
		I(\mathcal{G}_2;h_1) \leq \mathbb{E}_{p(\mathcal{G}_2,h_1)} \log q(h_1\vert \mathcal{G}_2)
		- \mathbb{E}_{p(\mathcal{G}_2)p(h_1)} \log q(h_1\vert \mathcal{G}_2)\\
		I(\mathcal{G}_1;h_2) \leq \mathbb{E}_{p(\mathcal{G}_1,h_2)} \log q(h_2\vert \mathcal{G}_1) - \mathbb{E}_{p(\mathcal{G}_1)p(h_2)} \log q(h_2\vert \mathcal{G}_1)
	\end{gather}
	where $q(h_1\vert \mathcal{G}_2)$ and $q(h_2\vert \mathcal{G}_1)$ are the variational approximations of posterior $p(h_1\vert \mathcal{G}_2)$ and $p(h_2\vert \mathcal{G}_1)$, respectively.
\end{proposition}

The upper bound of mutual information we utilize here is based on Contrastive Log-ratio Upper Bound (CLUB) \cite{club} with variational approximation, shown in Proposition \ref{pro2}.  Unfortunately, the variational CLUB is no longer guaranteed to remain an upper bound of mutual information under the variational approximation. Following the strategy in CLUB \cite{club}, we provide the detailed derivation of upper bound of $I(\mathcal{G}_2;h_1)$ and $I(\mathcal{G}_1;h_2)$. Taking $I(\mathcal{G}_1;h_2)$ as an example, we can define the contrastive log-ratio upper bound in form of 
\begin{gather} \label{appeq1}
	\mathbb{E}_{p(\mathcal{G}_1,h_2)} \log p(h_2\vert \mathcal{G}_1) - \mathbb{E}_{p(\mathcal{G}_1)p(h_2)} \log p(h_2\vert \mathcal{G}_1).
\end{gather} 
Then we can calculate the difference between $I(\mathcal{G}_1;h_2)$ and contrastive log-ratio upper bound in form of 
\begin{equation}
	\begin{aligned}
		\mathbb{E}&_{p(\mathcal{G}_1,h_2)} \log p(h_2\vert \mathcal{G}_1) - \mathbb{E}_{p(\mathcal{G}_1)p(h_2)} \log p(h_2\vert \mathcal{G}_1)  \\ &-\mathbb{E}_{p(\mathcal{G}_1,h_2)}\log \frac{p(\mathcal{G}_1,h_2)}{p(\mathcal{G}_1)p(h_2)}\\
		= \,&\mathbb{E}_{p(\mathcal{G}_1,h_2)} \log p(h_2\vert \mathcal{G}_1) - \mathbb{E}_{p(\mathcal{G}_1)p(h_2)} \log p(h_2\vert \mathcal{G}_1)  \\
		&- \mathbb{E}_{p(\mathcal{G}_1,h_2)}[\log p(h_2\vert \mathcal{G}_1) - \log p(h_2)]\\
		= \,&\mathbb{E}_{p(\mathcal{G}_1,h_2)} \log p(h_2) - \mathbb{E}_{p(\mathcal{G}_1)}\mathbb{E}_{p(h_2)}\log p(h_2\vert \mathcal{G}_1) \\
		= \,&\mathbb{E}_{p(h_2)}\left[ \log p(h_2) - \mathbb{E}_{p(\mathcal{G}_1)} \log p(h_2\vert \mathcal{G}_1) \right].
	\end{aligned}
\end{equation}
Note that the marginal distribution $p(\mathcal{G}_1)$ can be rewritten as 
\begin{gather}
p(h_2) = \int p(h_2\vert \mathcal{G}_1)p(\mathcal{G}_1)d\mathcal{G}_1.
\end{gather} 
According to Jensen's inequality, we have 
\begin{gather}
	\log p(h_2) = \log \mathbb{E}_{p(\mathcal{G}_1)} p(h_2\vert \mathcal{G}_1) \geq \mathbb{E}_{p(\mathcal{G}_1)} \log p(h_2\vert \mathcal{G}_1).
\end{gather}
Hence, Equation \ref{appeq1} is always non-negative and CLUB is an upper bound of $I(\mathcal{G}_2;h_1)$. Similarly, the upper bound of $I(\mathcal{G}_2;h_1)$ can be defined as  
\begin{gather}
	\mathbb{E}_{p(\mathcal{G}_2,h_1)} \log p(h_1\vert \mathcal{G}_2)
	- \mathbb{E}_{p(\mathcal{G}_2)p(h_1)} \log p(h_1\vert \mathcal{G}_2).
\end{gather}
However, $p(h_1\vert \mathcal{G}_2)$ and $p(h_1\vert \mathcal{G}_2)$ are unknown. We can utilize $q(h_1\vert \mathcal{G}_2)$ and $q(h_2\vert \mathcal{G}_1)$ to variationally approximate $p(h_1\vert \mathcal{G}_2)$ and $p(h_2\vert \mathcal{G}_1)$ in form of 
\begin{gather}
	\mathbb{E}_{p(\mathcal{G}_2,h_1)} \log q(h_1\vert \mathcal{G}_2)
	- \mathbb{E}_{p(\mathcal{G}_2)p(h_1)} \log q(h_1\vert \mathcal{G}_2)\\
	\mathbb{E}_{p(\mathcal{G}_1,h_2)} \log q(h_2\vert \mathcal{G}_1) - \mathbb{E}_{p(\mathcal{G}_1)p(h_2)} \log q(h_2\vert \mathcal{G}_1).
\end{gather}
Note that if   
\begin{gather}
	\mathcal{D}_{KL}(p(h_1,\mathcal{G}_2)\Vert q(h_1,\mathcal{G}_2)) \leq \mathcal{D}_{KL}(p(h_1)p(\mathcal{G}_2)\Vert q(h_1,\mathcal{G}_2)) \label{appeq13}\\ \mathcal{D}_{KL}(p(h_2,\mathcal{G}_1)\Vert q(h_2,\mathcal{G}_1))\leq\mathcal{D}_{KL}(p(h_2)p(\mathcal{G}_1)\Vert q(h_2,\mathcal{G}_1)) \label{appeq14}
\end{gather} 
then variational contrastive log-ratio upper bound can remain a mutual information upper bound. In practice, the variational distribution $q(h_1\vert \mathcal{G}_2)$ and $q(h_2\vert \mathcal{G}_1)$ are implemented with graph neural network. By enlarging the network capacity and maximizing the log-likelihood $\mathbb{E}_{p(h_1, \mathcal{G}_2)}\log q(h_1\vert \mathcal{G}_2)$ and $\mathbb{E}_{p(h_2, \mathcal{G}_1)}\log q(h_2\vert \mathcal{G}_1)$, we can obtain far more accurate approximation $q(h_1\vert \mathcal{G}_2)$ and $q(h_2\vert \mathcal{G}_1)$ for $p(h_1\vert \mathcal{G}_2)$ and $p(h_2\vert \mathcal{G}_1)$, respectively.

By combining the above lower and upper bounds, we can write the objective of achieving complementarity as 
\begin{equation} 
	\begin{aligned}
		\mathcal{L}_{h} =  
		&\mathbb{E}_{p(\mathcal{G}_1)}\mathbb{E}_{p(h_1\vert \mathcal{G}_1)}\log q(\mathcal{G}_1\vert h_1) + \mathbb{E}_{p(\mathcal{G}_2)}\mathbb{E}_{p(h_2\vert \mathcal{G}_2)}\log q(\mathcal{G}_2\vert h_2) \\
		&- \delta \left[ \mathbb{E}_{p(\mathcal{G}_2,h_1)} \log q(h_1\vert \mathcal{G}_2)
		- \mathbb{E}_{p(\mathcal{G}_2)p(h_1)} \log q(h_1\vert \mathcal{G}_2)\right] \\
		&- \eta \left[ \mathbb{E}_{p(\mathcal{G}_1,h_2)} \log q(h_2\vert \mathcal{G}_1) - \mathbb{E}_{p(\mathcal{G}_1)p(h_2)} \log q(h_2\vert \mathcal{G}_1) \right]
	\end{aligned}
	\label{eq_h}
\end{equation}  
where the first and the second terms are graph reconstruction loss, and the third and the fourth terms are the variational contrastive log-ratio upper bound loss. 

To enable the computation, we can utilize Monte Carlo sampling to approximately estimate and instantiate the complementarity objective. For upper bound terms in complementarity objective, we have 
\begin{gather}
	\mathbb{E}_{p(\mathcal{G}_1)}\mathbb{E}_{p(h_1\vert \mathcal{G}_1)}\log q(\mathcal{G}_1\vert h_1) \approx \frac{1}{N}\sum_{i=1}^{N}\mathbb{E}_{p(h_1 \vert \mathcal{G}_1^{i})}\log q(\mathcal{G}_1^{i} \vert h_1)\\
	\mathbb{E}_{p(\mathcal{G}_2)}\mathbb{E}_{p(h_2\vert \mathcal{G}_2)}\log q(\mathcal{G}_2\vert h_2) \approx 
	\frac{1}{N}\sum_{i=1}^{N}\mathbb{E}_{p(h_2 \vert \mathcal{G}_2^{i})}\log q(\mathcal{G}_2^{i}\vert h_2)
\end{gather}
where $N$ denotes the number of total sampled data. For contrastive log-ratio upper bound objective, we have
\begin{equation}
	\begin{aligned}
		&\delta \left[ \mathbb{E}_{p(\mathcal{G}_2,h_1)} \log q(h_1\vert \mathcal{G}_2)
		- \mathbb{E}_{p(\mathcal{G}_2)p(h_1)} \log q(h_1\vert \mathcal{G}_2)\right] \\
		&\approx \, \delta \left[ \frac{1}{N}\sum_{i=1}^{N}\bigg[ \log q(h_1^i\vert \mathcal{G}_2^i) - \frac{1}{N} \sum_{j=1}^{N} \log q(h_1^j\vert \mathcal{G}_2^i)\bigg]\right]
	\end{aligned}
\end{equation}
\begin{equation}
	\begin{aligned}
		&\eta \left[ \mathbb{E}_{p(\mathcal{G}_1,h_2)} \log q(h_2\vert \mathcal{G}_1) - \mathbb{E}_{p(\mathcal{G}_1)p(h_2)} \log q(h_2\vert \mathcal{G}_1) \right]\\
		&\approx \,\eta \left[ \frac{1}{N}\sum_{i=1}^{N}\bigg[ \log q(h_2^i\vert \mathcal{G}_1^i) - \frac{1}{N} \sum_{j=1}^{N} \log q(h_2^j\vert \mathcal{G}_1^i)\bigg]\right]
	\end{aligned}
\end{equation}
where $N$ denotes the number of total sampled data. We suppose the additional graph neural network encoder $q(h_1\vert \mathcal{G}_2)$ and $q(h_2\vert \mathcal{G}_1)$ are parameterized by Gaussian approximation $\mathcal{N}(h_1\vert (\mu(\mathcal{G}_2), \sigma^2(\mathcal{G}_2)))$ and $\mathcal{N}(h_2\vert (\mu(\mathcal{G}_1), \sigma^2(\mathcal{G}_1)))$ respectively, then for the given sample data $\{(h_1^i, \mathcal{G}_2^i)\}_{i=1}^{N}$ and $\{(h_2^i, \mathcal{G}_1^i)\}_{i=1}^{N}$, we have 
\begin{gather}
	\mu_1^i = \mu(\mathcal{G}_2^i), \, \sigma_1^i = \sigma(\mathcal{G}_2^i)\\
	\mu_2^i = \mu(\mathcal{G}_1^i), \, \sigma_2^i = \sigma(\mathcal{G}_1^i).
\end{gather}
Therefore, the CLUB objective can be calculated by
\begin{equation}
	\begin{aligned}
		\mathbb{E}&_{p(\mathcal{G}_2,h_1)} \log q(h_1\vert \mathcal{G}_2)
		- \mathbb{E}_{p(\mathcal{G}_2)p(h_1)} \log q(h_1\vert \mathcal{G}_2)=\\
		&-\frac{1}{2}\left\{ \frac{1}{N}\sum_{i=1}^{N}(h_1^i - \mu_1^i)^{\top}\text{Diag}[{(\sigma_1^i)}^{-2}](h_1^i - \mu_1^i) \right\} \\
		&+ \frac{1}{2} \left\{  \frac{1}{N^2}\sum_{i=1}^{N} \sum_{j=1}^{N} (h_1^j - \mu_1^i)^{\top}\text{Diag}[{(\sigma_1^i)}^{-2}](h_1^j - \mu_1^i) \right\}
	\end{aligned}
\end{equation}
\begin{equation}
	\begin{aligned}
		\mathbb{E}&_{p(\mathcal{G}_1,h_2)} \log q(h_2\vert \mathcal{G}_1)
		- \mathbb{E}_{p(\mathcal{G}_1)p(h_2)} \log q(h_2\vert \mathcal{G}_1)=\\
		&-\frac{1}{2}\left\{ \frac{1}{N}\sum_{i=1}^{N}(h_2^i - \mu_2^i)^{\top}\text{Diag}\big[{(\sigma_2^i)}^{-2}\big](h_2^i - \mu_2^i) \right\} \\
		&+ \frac{1}{2} \left\{  \frac{1}{N^2}\sum_{i=1}^{N} \sum_{j=1}^{N} (h_2^j - \mu_2^i)^{\top}\text{Diag}\big[{(\sigma_2^i)}^{-2}\big](h_2^j - \mu_2^i) \right\}
	\end{aligned}
\end{equation}
where $\text{Diag}[\cdot]$ denotes the diagonal matrix. 

\subsection{Towards Consistency}
Different from complementarity principle, consistency principle aims to extract the common representations $c_1$ and $c_2$ for the given views $\mathcal{G}_1$ and $\mathcal{G}_2$. Taking $f_{enc}: \mathcal{G}_1 \to c_1$ as an example, we expect to maximize the mutual information between $\mathcal{G}_1$ and $c_1$, while maximizing the mutual information between $\mathcal{G}_2$ and $c_1$, i.e., 
\begin{gather}
	\max I(\mathcal{G}_1;c_1) + \zeta I(\mathcal{G}_2;c_1) \label{e4}
\end{gather}
where $\zeta$ is a trade-off parameter to balance $I(\mathcal{G}_1;c_1)$ and $I(\mathcal{G}_2;c_1)$. Analogously, we can formulate the information objective for $f_{enc}: \mathcal{G}_2 \to c_2$, i.e., 
\begin{gather}
	\max I(\mathcal{G}_2;c_2) + \xi I(\mathcal{G}_1;c_2) \label{e5}
\end{gather}
where $\xi$ is the trade-off parameter. Intuitively, maximizing Equation \ref{e4} and Equation \ref{e5} encourage the encoders to maximally express the multiple input views simultaneously, thus being able to model the common information between multiviews, i.e., towards consistency. However, the exact computation of mutual information is still intractable. We follow the variational inference approach and derive the lower bounds of these mutual information terms, as shown in Proposition \ref{pro3} and Proposition \ref{pro4}.
\begin{proposition}[Lower bounds of $I(\mathcal{G}_1;c_1)$ and $I(\mathcal{G}_2;c_2)$] \label{pro3}
	For mutual information $I(\mathcal{G}_1;c_1)$ and $I(\mathcal{G}_2;c_2)$, we have
	\begin{gather}
		I(\mathcal{G}_1;c_1) \geq \mathbb{E}_{p(\mathcal{G}_1)}\mathbb{E}_{p(c_1\vert \mathcal{G}_1)}\log q(\mathcal{G}_1\vert c_1)\\
		I(\mathcal{G}_2;c_2) \geq \mathbb{E}_{p(\mathcal{G}_2)}\mathbb{E}_{p(c_2\vert \mathcal{G}_2)}\log q(\mathcal{G}_2\vert c_2)
	\end{gather}
	where $q(\mathcal{G}_1\vert c_1)$ and $q(\mathcal{G}_2\vert c_2)$ are the variational approximations of the posterior $p(\mathcal{G}_1\vert c_1)$ and $p(\mathcal{G}_2\vert c_2)$, respectively.
\end{proposition}

Following the proof of Proposition \ref{pro1}, we can easily obtain the variational lower bounds of $I(\mathcal{G}_1;c_1)$ and $I(\mathcal{G}_2;c_2)$ in a similar way, thus derive Proposition \ref{pro3}.

\begin{proposition}[Lower bounds of $I(\mathcal{G}_2;c_1)$ and $I(\mathcal{G}_1;c_2)$] \label{pro4}
	For mutual information $I(\mathcal{G}_2;c_1)$ and $I(\mathcal{G}_1;c_2)$, we have
	\begin{gather}
		I(\mathcal{G}_2;c_1) \geq \mathbb{E}_{p(\mathcal{G}_1,\mathcal{G}_2)}\mathbb{E}_{p(c_1\vert \mathcal{G}_1)}\log q(\mathcal{G}_2\vert c_1)\\
		I(\mathcal{G}_1;c_2) \geq \mathbb{E}_{p(\mathcal{G}_1,\mathcal{G}_2)}\mathbb{E}_{p(c_2\vert \mathcal{G}_2)}\log q(\mathcal{G}_1\vert c_2)
	\end{gather}
	where $q(\mathcal{G}_2\vert c_1)$ and $q(\mathcal{G}_1\vert c_2)$ are the variational approximations of the posterior $p(\mathcal{G}_2\vert c_1)$ and $p(\mathcal{G}_1\vert c_2)$, respectively.
\end{proposition}

\begin{proof}
	Let $q(\mathcal{G}_1\vert c_2)$ be a variational approximation to $p(\mathcal{G}_1\vert c_2)$. Using the fact that Kullback-Leiber divergence is always non-negative, we have
	\begin{gather}
		\mathcal{D}_{KL}(p(\mathcal{G}_1\vert c_2)\Vert q(\mathcal{G}_1\vert c_2)) \geq 0 \\
		\mathbb{E}_{p(\mathcal{G}_1\vert c_2)} \log \frac{p(\mathcal{G}_1\vert c_2)}{q(\mathcal{G}_1\vert c_2)} \geq 0 \\
		\mathbb{E}_{p(\mathcal{G}_1\vert c_2)} \log p(\mathcal{G}_1\vert c_2) \geq \mathbb{E}_{p(\mathcal{G}_1\vert c_2)} \log q(\mathcal{G}_1\vert c_2).
	\end{gather}
	Then, according to the definition of mutual information, we can derive 
	\begin{equation}
		\begin{aligned}
			I(\mathcal{G}_1;c_2) &\geq \mathbb{E}_{p(\mathcal{G}_1,c_2)}\log \frac{q(\mathcal{G}_1\vert c_2)}{p(\mathcal{G}_1)} \\
			&= \mathbb{E}_{p(\mathcal{G}_1,c_2)}\log q(\mathcal{G}_1\vert c_2) - \mathbb{E}_{p(\mathcal{G}_1)}\log p(\mathcal{G}_1)\\
			&= \mathbb{E}_{p(\mathcal{G}_1,c_2)}\log q(\mathcal{G}_1\vert c_2) + H(\mathcal{G}_1).
		\end{aligned}
	\end{equation}
	Note that the entropy $H(\mathcal{G}_1)$ is independent of the objective optimization and can be ignored, hence we obtain
	\begin{gather} \label{eq27a}
		I(\mathcal{G}_1,c_2) \geq \mathbb{E}_{p(\mathcal{G}_1,c_2)}\log q(\mathcal{G}_1\vert c_2). 
	\end{gather} 
	We now consider the joint distribution $p(\mathcal{G}_1,\mathcal{G}_2)$ and rewrite this distribution in form of
	\begin{equation}
		\begin{aligned}
			p(\mathcal{G}_1,c_2) &= \int p(\mathcal{G}_1, \mathcal{G}_2, c_2)d\mathcal{G}_2 \\
			&= \int p(\mathcal{G}_1,\mathcal{G}_2)p(c_2\vert (\mathcal{G}_1,\mathcal{G}_2))d\mathcal{G}_2\\
			&= \int p(\mathcal{G}_1,\mathcal{G}_2)p(c_2\vert \mathcal{G}_2)d\mathcal{G}_2.
		\end{aligned} \label{eq28}
	\end{equation}
	Therefore, by replacing the joint distribution $p(\mathcal{G}_1,c_2)$ in Equation \ref{eq27a}, we have
	\begin{gather}
		I(\mathcal{G}_1;c_2) \geq \mathbb{E}_{p(\mathcal{G}_1,\mathcal{G}_2)}\mathbb{E}_{p(c_2\vert \mathcal{G}_2)}\log q(\mathcal{G}_1\vert c_2).
	\end{gather}
	Similarly, we can derive the lower bound of $I(\mathcal{G}_2;c_1)$ as 
	\begin{gather} 
		I(\mathcal{G}_2;c_1) \geq \mathbb{E}_{p(\mathcal{G}_1,\mathcal{G}_2)}\mathbb{E}_{p(c_1\vert \mathcal{G}_1)}\log q(\mathcal{G}_2\vert c_1).
	\end{gather}
\end{proof} 

By combining the above lower bounds, we can write the objective of achieving consistency as 
\begin{equation}
	\begin{aligned}
		\mathcal{L}_{c} = \,\, &\mathbb{E}_{p(\mathcal{G}_1)}\mathbb{E}_{p(c_1\vert \mathcal{G}_1)}\log q(\mathcal{G}_1\vert c_1) \\
		&+ \mathbb{E}_{p(\mathcal{G}_2)}\mathbb{E}_{p(c_2\vert \mathcal{G}_2)}\log q(\mathcal{G}_2\vert c_2) \\
		&+ \zeta \cdot \mathbb{E}_{p(\mathcal{G}_1,\mathcal{G}_2)}\mathbb{E}_{p(c_1\vert \mathcal{G}_1)}\log q(\mathcal{G}_2\vert c_1) \\
		&+ \xi \cdot \mathbb{E}_{p(\mathcal{G}_1,\mathcal{G}_2)}\mathbb{E}_{p(c_2\vert \mathcal{G}_2)}\log q(\mathcal{G}_1\vert c_2)
	\end{aligned}
	\label{eq_c}
\end{equation}
where the first/second and the third/fourth terms encourage that the common representation $c_1$ and $c_2$ can simultaneously reconstruct the input graph view $\mathcal{G}_1$ and $\mathcal{G}_2$. We can also utilize Monte Carlo sampling to approximately estimate and instantiate the consistency objective, and can obtain 
\begin{gather}
	\mathbb{E}_{p(\mathcal{G}_1)}\mathbb{E}_{p(c_1\vert \mathcal{G}_1)}\log q(\mathcal{G}_1\vert c_1) \approx \frac{1}{N}\sum_{i=1}^{N}\mathbb{E}_{p(c_1 \vert \mathcal{G}_1^{i})}\log q(\mathcal{G}_1^{i} \vert c_1)
\end{gather}
\begin{gather}
	\mathbb{E}_{p(\mathcal{G}_2)}\mathbb{E}_{p(c_2\vert \mathcal{G}_2)}\log q(\mathcal{G}_2\vert c_2) \approx  \frac{1}{N}\sum_{i=1}^{N}\mathbb{E}_{p(c_2 \vert \mathcal{G}_2^{i})}\log q(\mathcal{G}_2^{i} \vert c_2)
\end{gather}
\begin{equation}
	\begin{aligned}
		\zeta \cdot \mathbb{E}&_{p(\mathcal{G}_1,\mathcal{G}_2)}\mathbb{E}_{p(c_1\vert \mathcal{G}_1)}\log q(\mathcal{G}_2\vert c_1) \\
		&\approx \,  
		\zeta \cdot \left[ \frac{1}{N}\sum_{i=1}^{N} \mathbb{E}_{p(c_1 \vert \mathcal{G}_1^{i})}\log q(\mathcal{G}_2^{i} \vert c_1) \right]
	\end{aligned}
\end{equation}
\begin{equation}
	\begin{aligned}
		\xi \cdot
		\mathbb{E}&_{p(\mathcal{G}_1,\mathcal{G}_2)}\mathbb{E}_{p(c_2\vert \mathcal{G}_2)}\log q(\mathcal{G}_1\vert c_2) \\ 
		&\approx \, 
		\xi \cdot \left[ \frac{1}{N}\sum_{i=1}^{N} \mathbb{E}_{p(c_2\vert \mathcal{G}_2^{i})}\log q(\mathcal{G}_1^{i}\vert c_2) \right]
	\end{aligned}
\end{equation}
where $N$ denotes the number of total sampled data.

\subsection{Optimization and Inference}

After obtaining the objectives of achieving complementarity and consistency, we further combine these two information objectives to derive the overall objective in form of 
\begin{gather}
	\min \mathcal{L} = -(\mathcal{L}_{h} + \mathcal{L}_{c})
\end{gather}
where $\mathcal{L}_{h}$ and $ \mathcal{L}_{c}$ are complementarity objective (Equation \ref{eq_h}) and consistency objective (Equation \ref{eq_c}), respectively. To optimize this objective, we employ the classical graph neural network to encode the input graph data into the latent representations $\{h_1, h_2, c_1, c_2\}$. To define the variational approximations $\{q(\mathcal{G}_1\vert h_1), q(\mathcal{G}_2\vert h_2), q(\mathcal{G}_1\vert c_1), q(\mathcal{G}_2\vert c_2)\}$, we utilize the graph decoder \cite{vgae} to reconstruct the adjacency matrix $A_1$ and $A_2$, and the multilayer perception to reconstruct the node feature $X$, since input $\mathcal{G}$ contains topology and node feature information. Note that minimizing the upper bound of $I(\mathcal{G}_2;h_1)$ and $I(\mathcal{G}_1;h_2)$ need to get the variational conditional distributions $q(h_1\vert \mathcal{G}_2)$ and $q(h_2\vert \mathcal{G}_1)$, but these distributions are unknown. Here, we use an additional graph neural network to encode $\mathcal{G}_1$ and $\mathcal{G}_2$ to $h_2$ and $h_1$ for defining $q(h_1\vert \mathcal{G}_2)$ and $q(h_2\vert \mathcal{G}_1)$, respectively. 

The overall objective can be easily optimized by using stochastic gradient descent, and at inference time, the concatenation function is employed to aggregate latent representations $c_1$, $c_2$, $h_1$, $h_2$ for downstream tasks. 

\section{Experiments} \label{s5}
To verify the effectiveness, we evaluate the proposed approach on the graph classification and graph clustering tasks, and also perform the ablation study. 

\begin{table}[b]
	\centering
	\caption{Statistics of graph benchmark datasets.}
	\scalebox{0.85}{
		\begin{tabular}{|c|c|c|c|c|}
			\toprule
			\textsc{Dataset}&Graphs&Avg. Nodes&Avg. Edges& Classes\\
			\midrule
			\midrule
			MUTAG&188&17.93&19.79&2\\
			PROTEINS&1,113&39.06&72.82&2\\
			PTC-MR&344&14.29&14.69&2\\
			IMDB-BINARY&1,000&19.77&96.53&2\\
			IMDB-MULTI&1,000&13.00&65.94&3\\
			REDDIT-BINARY&2,000&429.63&497.75&2\\
			REDDIT-M5K&4,999&508.52&594.87&5\\
			\bottomrule
	\end{tabular}}
	\label{apptab1}
\end{table}

%\begin{table}[h]
%	\centering
%	\caption{Parameter setting for the graph representation experiments.}
%	\scalebox{0.8}{
%		\begin{tabular}{|c|c|}
%			\toprule
%			Parameter Name&Value\\
%			\midrule
%			\midrule
%			Number of layers of GIN encoder&$5$\\
%			\midrule
%			Number of layers of feature reconstruction decoder&$4$\\
%			\midrule
%			Number of encoder layers of CLUB decoder&$4$\\
%			\midrule
%			Number of hidden&$64$\\
%			\midrule
%			Batch size&$128$\\
%			\midrule
%			Total epochs&$100$\\
%			\midrule
%			Initial learning rate&$0.01$\\
%			\midrule
%			Lr decay factor&$0.5$\\
%			\midrule
%			Lr decay step size&$50$\\
%			\midrule
%			$K$ in KNN&$5$\\
%			\midrule
%			$D$ in DKNN&$2$\\
%			\midrule
%			$\alpha$ in PPR&$0.2$\\
%			\bottomrule
%	\end{tabular}}
%	\label{apptab0}
%\end{table}

\begin{table*}[t]
	\centering
	\caption{Graph classification experimental results on benchmark datasets. The best performance is highlighted by the bold number and the underlined numbers denote the second-best performance. \textsc{Sup.} and \textsc{Unsup.} denote the Supervised and Unsupervised, respectively, and OOM denotes the Out of Memory. }
	\scalebox{0.86}{
		\begin{tabular}{|c|c|c|c|c|c|c|c|c|}
			\toprule
			&\textsc{Methods}&MUTAG&PROTEINS&PTC-MR&IMDB-BINARY&IMDB-MULTI&REDDIT-BINARY&REDDIT-M5K\\
			\midrule
			\midrule
			\multirow{4}*{\rotatebox{90}{\textsc{Kernel}}}&SP&83.54 $\pm$ 7.19&64.96 $\pm$ 3.43&59.96 $\pm$ 7.13&71.30 $\pm$ 4.36&47.53 $\pm$ 2.07&64.11 $\pm$ 0.14 &39.55 $\pm$ 0.22\\ 
			&RW&84.62 $\pm$ 6.81&$>$ 1 day&60.18 $\pm$ 8.52&50.68 $\pm$ 0.26&34.65 $\pm$ 0.19&$>$ 1 day&$>$ 1 day\\
			&WL&85.12 $\pm$ 6.97&73.32 $\pm$ 3.08&59.94 $\pm$ 7.53&72.30 $\pm$ 3.44&46.95 $\pm$ 0.46&73.05 $\pm$ 3.86&40.33 $\pm$ 1.92\\
			&Graphlet&85.61 $\pm$ 4.87&64.52 $\pm$ 4.25&61.04 $\pm$ 9.39&65.87 $\pm$ 0.98&43.89 $\pm$ 0.38&77.34 $\pm$ 0.18&41.01 $\pm$ 0.17\\
			\midrule
			\multirow{3}*{\rotatebox{90}{\textsc{Sup.}}}&GIN&88.14 $\pm$ 6.38&74.13 $\pm$ 4.21&61.61 $\pm$ 4.62&73.67 $\pm$ 3.67&50.87 $\pm$ 2.80&85.40 $\pm$ 3.12&52.50 $\pm$ 2.10\\
			&KGIN&85.15 $\pm$ 6.92&75.65 $\pm$ 3.59&61.28 $\pm$ 4.58&68.80 $\pm$ 6.32&50.60 $\pm$ 1.97&76.95 $\pm$ 4.13&49.61 $\pm$ 7.09\\
			&SIB&\underline{89.36 $\pm$ 7.45}&74.97 $\pm$ 5.71
			&62.12 $\pm$ 6.14&73.30 $\pm$ 3.90&51.37 $\pm$ 4.37&85.20 $\pm$ 3.93&OOM\\
			\midrule
			\multirow{5}*{\rotatebox{90}{\textsc{Unsup.}}}&GAE&87.81 $\pm$ 8.47&72.05 $\pm$ 3.56&60.73 $\pm$ 5.73&71.90 $\pm$ 4.70&47.60 $\pm$ 2.83&80.15 $\pm$ 1.87&47.87 $\pm$ 1.96\\
			&VGAE&87.75 $\pm$ 5.34&73.86 $\pm$ 2.31&61.39 $\pm$ 6.31&70.00 $\pm$ 3.92&46.40 $\pm$ 5.01&\underline{85.75 $\pm$ 2.69}&49.25 $\pm$ 2.10\\
			&DGI&87.72 $\pm$ 9.24&73.05 $\pm$ 2.79&61.36 $\pm$ 4.82&69.50 $\pm$ 2.46&48.40 $\pm$ 2.53&85.60 $\pm$ 1.79&49.95 $\pm$ 1.49\\
			&InfoGraph&88.33 $\pm$ 6.54&73.23 $\pm$ 6.59&61.04 $\pm$ 7.22&72.80 $\pm$ 4.09&48.07 $\pm$ 4.03&85.35 $\pm$ 2.67&48.89 $\pm$ 2.23\\
			&MVGRL&85.70 $\pm$ 7.77&\underline{75.65 $\pm$ 2.84}&\underline{62.90 $\pm$ 7.41}&\underline{74.20 $\pm$ 5.49}&\underline{50.80 $\pm$ 3.53}&83.35 $\pm$ 2.66&\underline{52.63 $\pm$ 1.82}\\
			\midrule
			&MVGIB&\textbf{91.49 $\pm$ 4.83}&\textbf{76.46 $\pm$ 3.08}&\textbf{65.57 $\pm$ 6.90}&\textbf{76.70 $\pm$ 5.06}&\textbf{51.93 $\pm$ 2.67}&\textbf{90.65 $\pm$ 1.47}&\textbf{54.01 $\pm$ 1.99}\\
			\bottomrule
	\end{tabular}}
	\label{tab11}
\end{table*}

\begin{table*}[t]
	\centering
	\caption{Impact of feature reconstruction, complementarity or consistency, and different views.}
	\scalebox{0.83}{
		\begin{tabular}{|c|c|c|c|c|c|c|c|}
			\toprule
			\textsc{Ablation Study}&MUTAG&PROTEINS&PTC-MR&IMDB-BINARY&IMDB-MULTI&REDDIT-BINARY&REDDIT-M5K\\
			\midrule
			\midrule
			w/o Fea. Rescon.&90.54 $\pm$ 5.99&76.28 $\pm$ 1.70&63.12 $\pm$ 7.69&75.30 $\pm$ 3.87&50.27 $\pm$ 2.69&90.25 $\pm$ 1.15&52.67 $\pm$ 2.01\\
			w/ Fea. Rescon.&\textbf{91.49 $\pm$ 4.83}&\textbf{76.46 $\pm$ 3.08}&\textbf{65.67 $\pm$ 6.90}&\textbf{76.70 $\pm$ 5.06}&\textbf{51.93 $\pm$ 2.67}&\textbf{90.65 $\pm$ 1.47}&\textbf{54.01 $\pm$ 1.99}\\
			\midrule
			w/o Complementarity&89.94 $\pm$ 4.85&75.83 $\pm$ 2.93&63.97 $\pm$ 5.40&74.60 $\pm$ 3.80&50.07 $\pm$ 4.04&89.40 $\pm$ 2.13&51.39 $\pm$ 1.76\\
			w/o Consistency&84.56 $\pm$ 5.06&74.04 $\pm$ 3.09&60.48 $\pm$ 7.33&73.90 $\pm$ 4.70&48.20 $\pm$ 4.30&80.65 $\pm$ 3.99&50.73 $\pm$ 1.84\\
			w/ Com. and Con.&\textbf{91.49 $\pm$ 4.83}&\textbf{76.46 $\pm$ 3.08}&\textbf{65.67 $\pm$ 6.90}&\textbf{76.70 $\pm$ 5.06}&\textbf{51.93 $\pm$ 2.67}&\textbf{90.65 $\pm$ 1.47}&\textbf{54.01 $\pm$ 1.99}\\
			\midrule
			\textsc{Adj-Knn}&\textbf{91.49 $\pm$ 4.83}&\textbf{76.46 $\pm$ 3.08}&62.48 $\pm$ 5.59&\textbf{76.70 $\pm$ 5.06}&50.20 $\pm$ 2.35&90.55 $\pm$ 1.59&\textbf{54.01 $\pm$ 1.99}\\
			\textsc{Adj-Dknn}&91.46 $\pm$ 5.42&75.11 $\pm$ 3.57&\textbf{65.67 $\pm$ 6.90}&75.00 $\pm$ 3.85&\textbf{51.93 $\pm$ 2.67}&\textbf{90.65 $\pm$ 1.47}&53.43 $\pm$ 2.57\\
			\textsc{Adj-Ppr}&89.36 $\pm$ 6.28&74.93 $\pm$ 4.75&63.41 $\pm$ 7.40&73.50 $\pm$ 3.23&50.93 $\pm$ 3.45&85.75 $\pm$ 2.75 &52.93 $\pm$ 1.39\\
			\textsc{Knn-Dknn}&88.36 $\pm$ 4.51&75.56 $\pm$ 2.94&62.81 $\pm$ 8.22&74.20 $\pm$ 4.58&51.53 $\pm$ 4.20&84.30 $\pm$ 2.39&51.59 $\pm$ 1.64\\
			\textsc{Knn-Ppr}&88.89 $\pm$ 7.19&73.12 $\pm$ 4.48&62.19 $\pm$ 5.46&72.70 $\pm$ 2.97&51.27 $\pm$ 3.05&85.53 $\pm$ 1.43&50.91 $\pm$ 1.29\\
			\textsc{Ppr-Dknn}&86.17 $\pm$ 5.38&72.05 $\pm$ 4.59&62.76 $\pm$ 9.04&72.80 $\pm$ 5.58&49.93 $\pm$ 4.21&87.00 $\pm$ 2.55&49.57 $\pm$ 1.62\\
			\bottomrule
	\end{tabular}}
	\label{tab22}
\end{table*}

\subsection{Datasets}
We evaluate our method on seven graph benchmark \mbox{datasets}, including MUTAG \cite{2012Subgraph}, PROTEINS \cite{2005Protein}, PTC-MR \cite{2012Subgraph}, IMDB-BINARY (IMDB-B), IMDB-MULTI (IMDB-M), REDDIT-BINARY (REDDIT-B), and REDDIT-MULTI-5K (REDDIT-M5K) \cite{dgk}. Here, MUTAG is a mutagenic aromatic and heteroaromatic nitro compounds dataset and their graph label indicate whether the mutagenic \mbox{effect} on bacteria exists; PROTEINS represents protein structures which are helix, sheet, and turn; PTC-MR is a chemical compounds dataset which represents the carcinogenicity for male and female rats; IMDB-BINARY and IMDB-MULTI are movie collaboration datasets. Each graph corresponds to an ego-network for each actor/actress, where nodes correspond to actors/actresses and an edge is drawn between two actors/actresses if they appear in the same movie; REDDIT-MULTI and REDDIT-M5K are balanced datasets where each graph corresponds to an online discussion thread and nodes correspond to users. For IMDB and REDDIT datasets, we use the one-hot encoding of node degrees as node features since the original data has no node features. The dataset statistics are shown in Table \ref{apptab1}.

\begin{table}[tb]
	\centering
	\caption{Graph clustering results on benchmark datasets. The best performance is highlighted by the bold number and the underlined numbers denote the second-best performance.}
	\scalebox{0.8}{
		\begin{tabular}{|c|c|c|c|c|c|}
			\toprule
			\textsc{Datasets}&\textsc{Metric}&DGI&InfoGraph&MVGRL&MVGIB\\
			\midrule
			\midrule
			\multirow{2}*{IMDB-B}&NMI&0.0801&\underline{0.0811}&0.0804&\textbf{0.0821}\\
			&ARI&0.0102&0.0112&0.0276&\textbf{0.0314}\\
			\midrule
			\multirow{2}*{IMDB-M}&NMI&0.0332&\textbf{0.0378}&0.0247&\underline{0.0367}\\
			&ARI&0.0013&0.0012&\underline{0.0109}&\textbf{0.0112}\\
			\midrule
			\multirow{2}*{REDDIT-B}&NMI&0.0142&0.0908&\underline{0.1392}&\textbf{0.1517}\\
			&ARI&0.0002&0.0258&\underline{0.1515}&\textbf{0.1525}\\
			\midrule
			\multirow{2}*{REDDIT-M5K}&NMI&0.0691&0.0142&\underline{0.1510}&\textbf{0.1709}\\
			&ARI&0.0527&0.0003&\underline{0.1150}&\textbf{0.1318}\\
			\bottomrule
	\end{tabular}}
	\label{tab3}
\end{table}

\subsection{Baselines}
We compare the proposed method with previous graph-level representation learning baselines, including graph kernel approaches and supervised or unsupervised graph neural network approaches. The graph kernel approaches include Shortest Path (SP) kernel \cite{spk}, Random Walk (RW) kernel \cite{rwk}, Weisfeiler-Lehman (WL) kernel \cite{wlk}, and Graphlet kernel \cite{graphlet}. The supervised graph neural network approaches include Graph Isomorphism Network (GIN) \cite{gin}, GIN with K-Nearest Neighbors (KGIN), and  Subgraph Information Bottleneck (SIB) \cite{sib}. The unsupervised graph neural network approaches include Graph Auto-Encoder (GAE), Variational GAE \cite{vgae}, Deep Graph Infomax (DGI) \cite{dgi}, InfoGraph \cite{infograph}, and Contrastive Multiview Graph Representation Learning (MVGRL) \cite{mvgrl}.

\begin{figure*}[t]
	\centering
	\subfigure[PROTEINS]{\includegraphics[width=0.57\columnwidth]{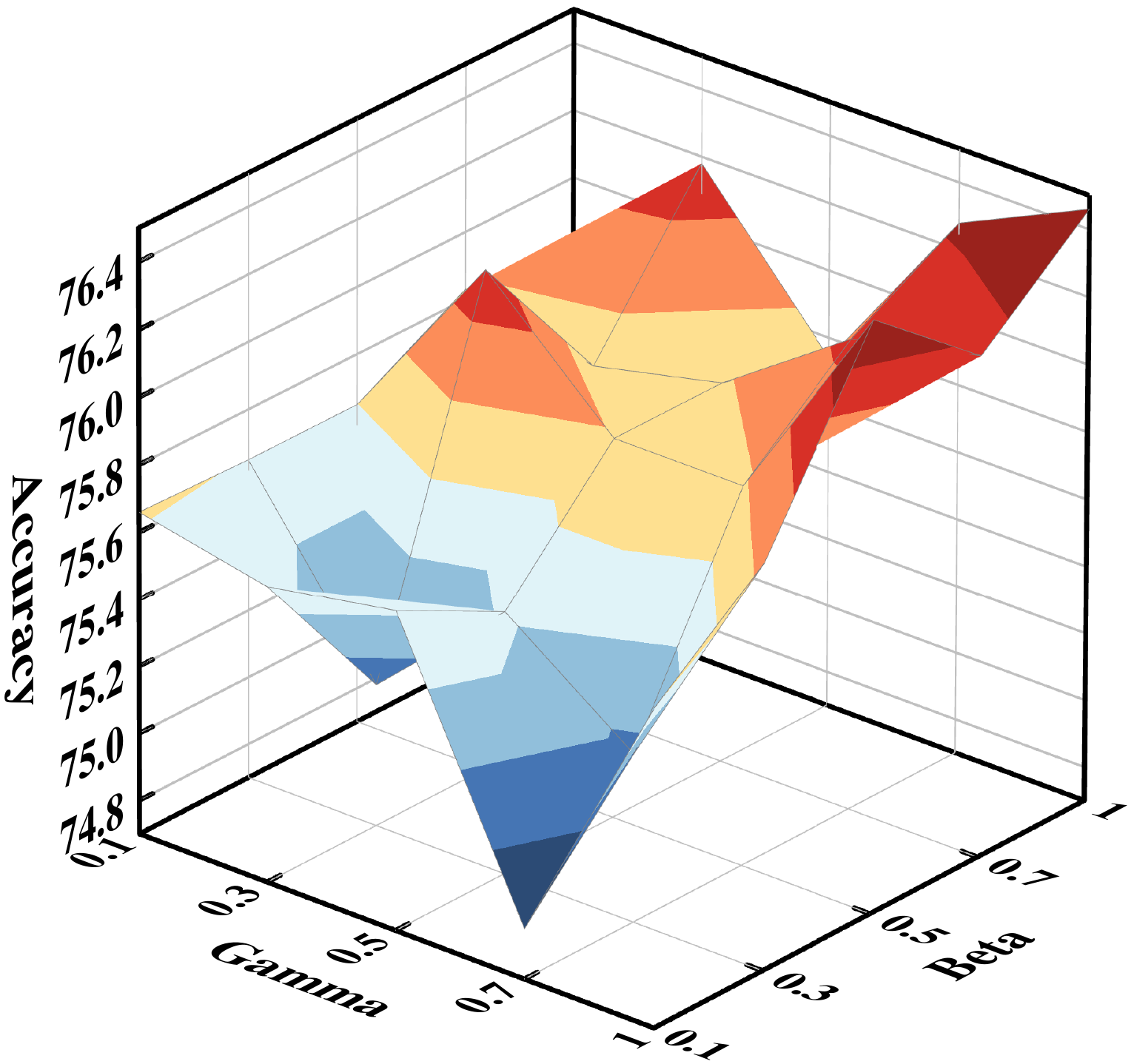}} \quad
	\subfigure[PTC-MR]{\includegraphics[width=0.565\columnwidth]{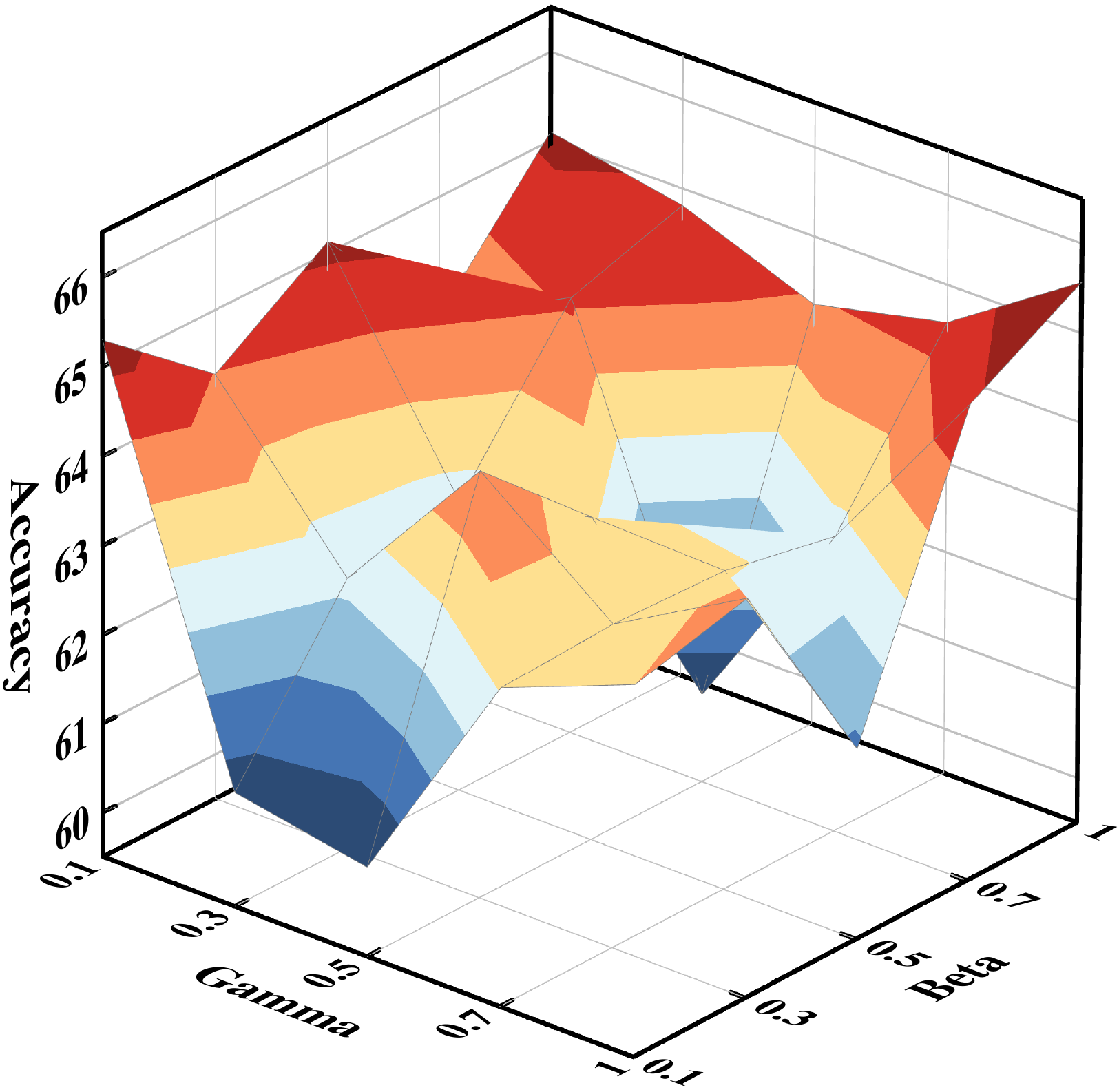}}  \quad
	\subfigure[IMDB-BINARY]{\includegraphics[width=0.57\columnwidth]{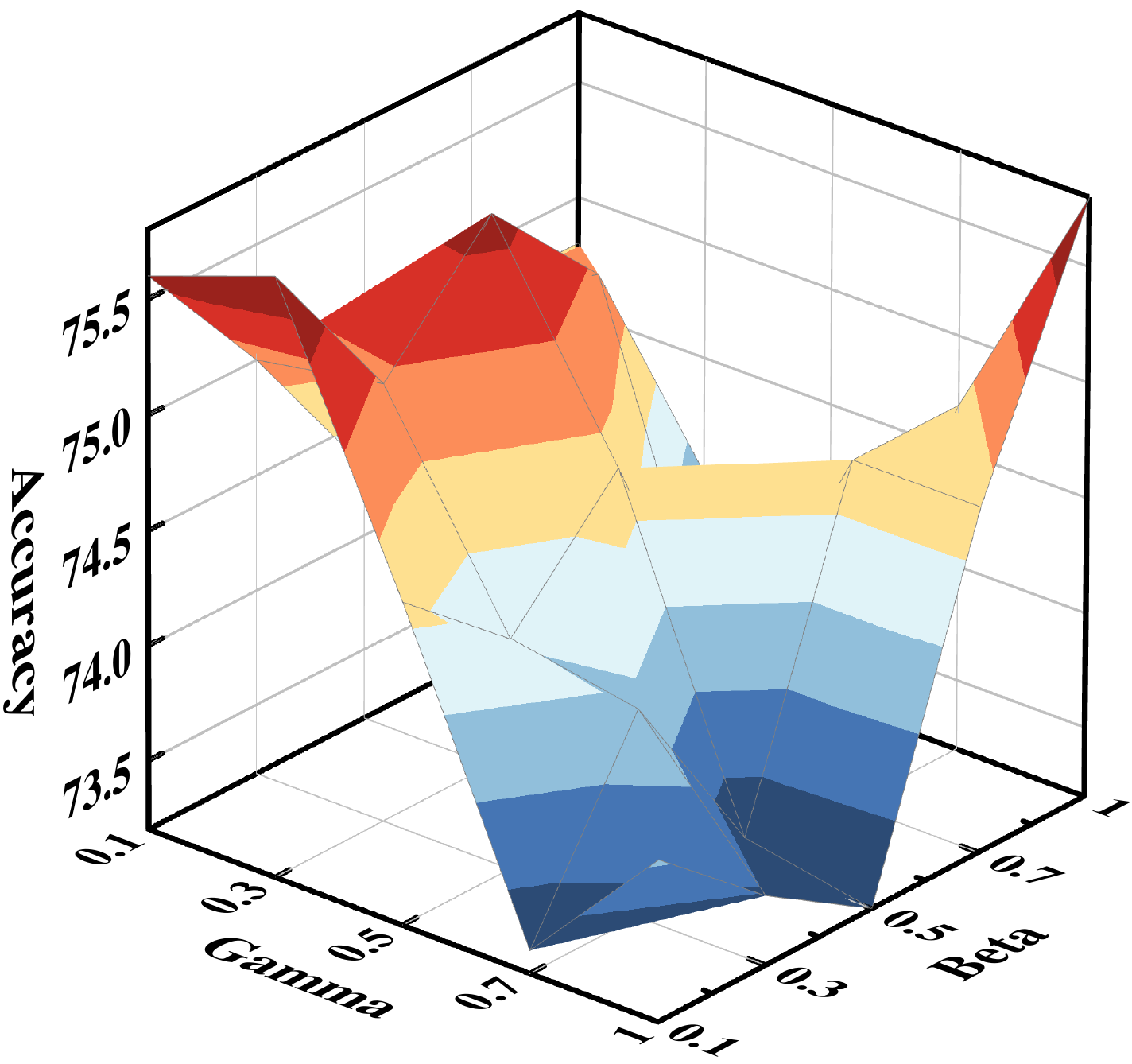}} \quad
	\subfigure[IMDB-MULTI]{\includegraphics[width=0.57\columnwidth]{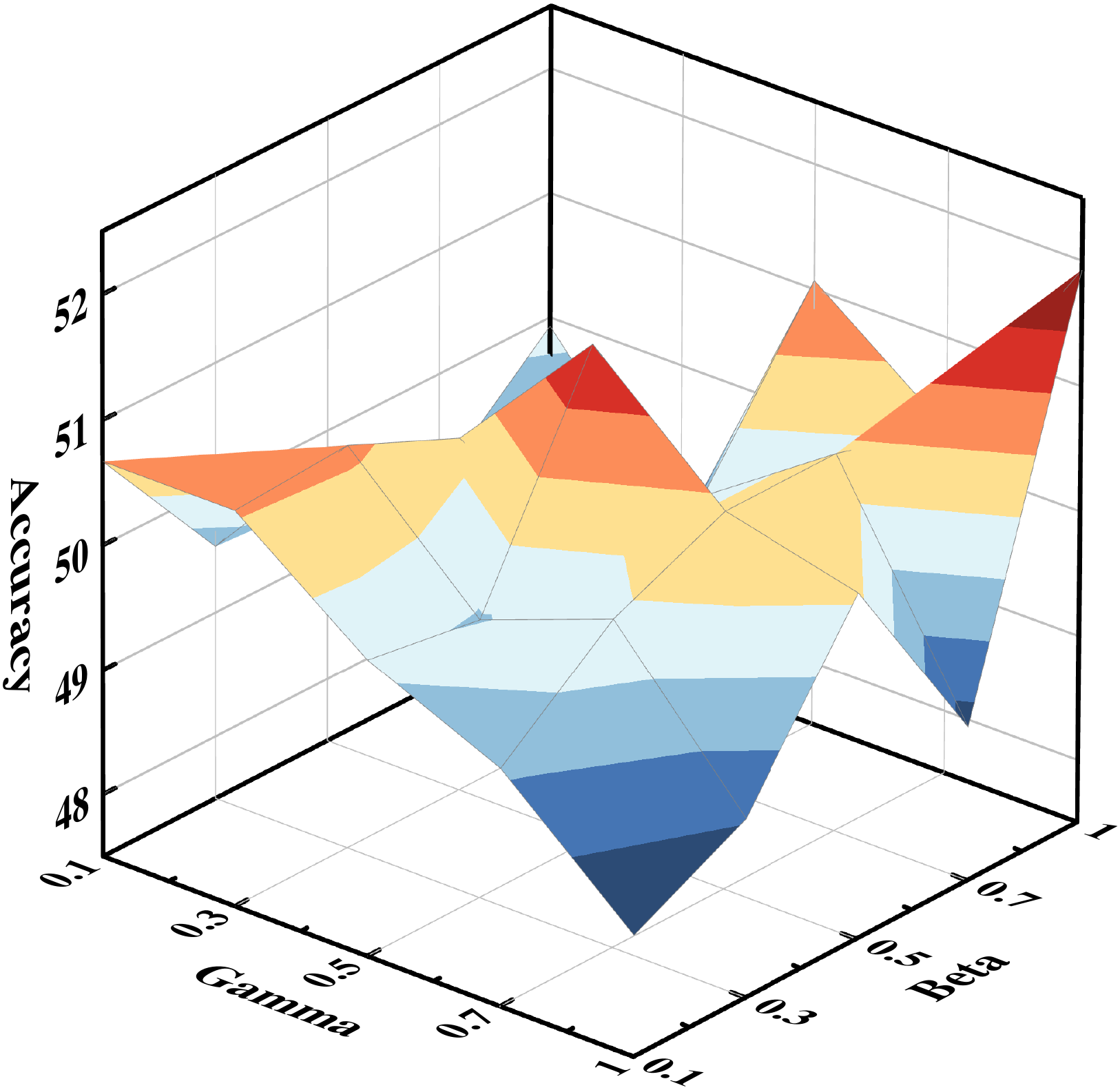}}  \quad
	\subfigure[REDDIT-BINARY]{\includegraphics[width=0.57\columnwidth]{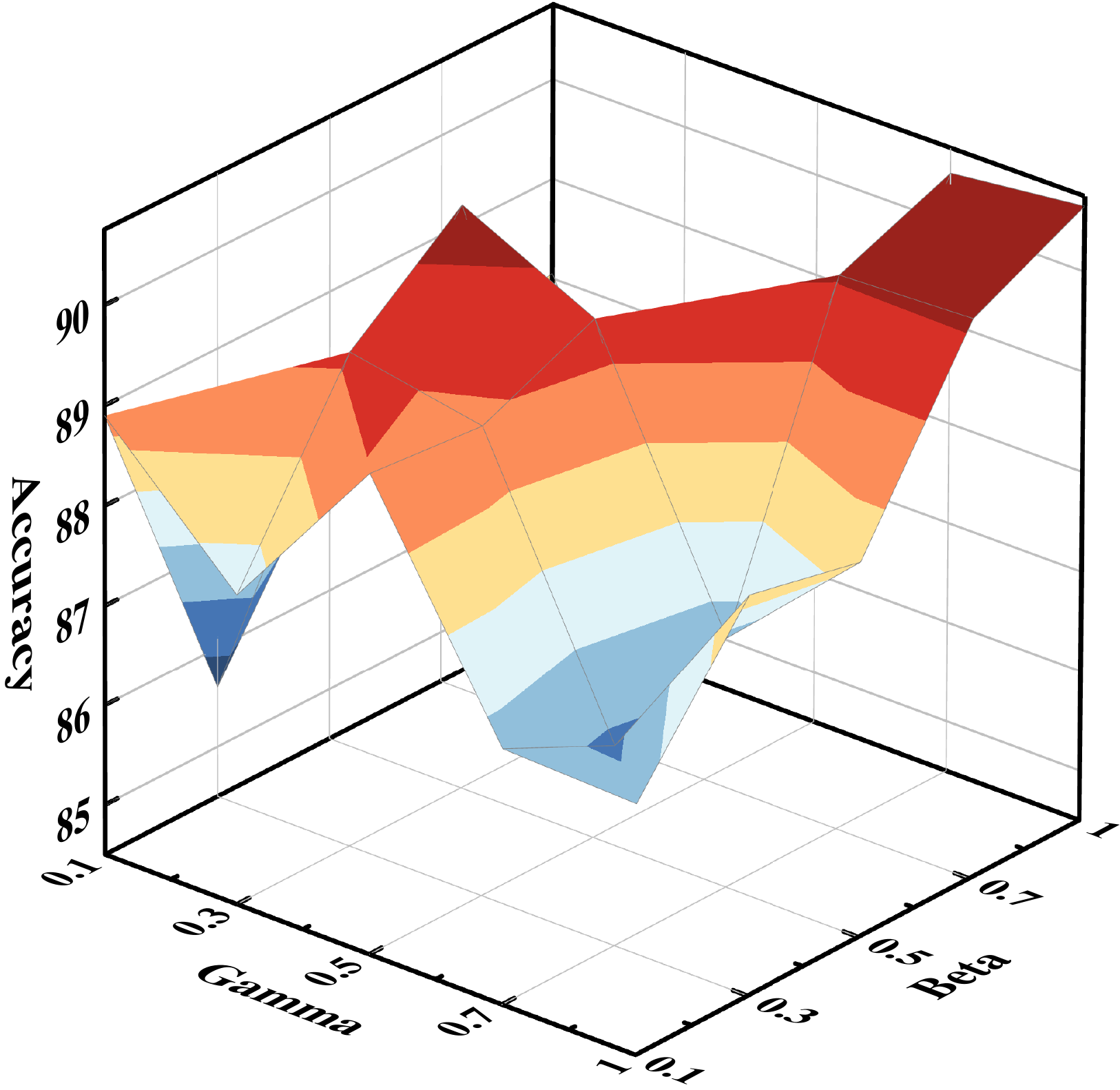}} \quad
	\subfigure[REDDIT-M5K]{\includegraphics[width=0.57\columnwidth]{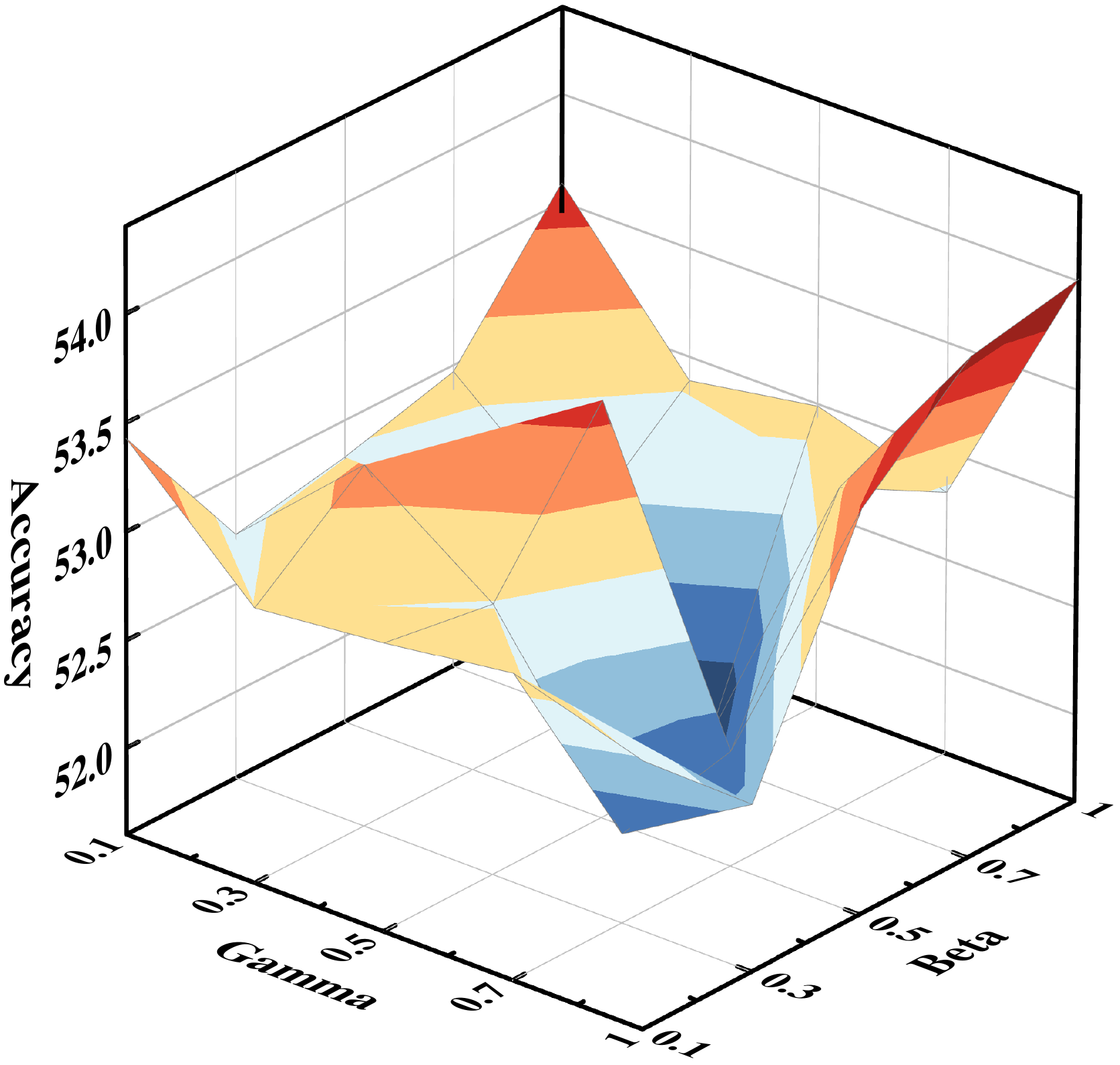}} \quad
	\caption{Impact of Trade-off Parameters on PROTEINS, PTC-MR, IMDB-BINARY, IMDB-MULTI, REDDIT-BINARY, and REDDIT-M5K datasets. The proposed MVGIB achieves better performance when $\beta$ and $\gamma$ are close or even equal, and vice versa. }
	\label{graph1}
\end{figure*}

\subsection{Configuration}
Thanks to the powerful representation ability of GIN \cite{gin}, we utilize GIN as the base encoder to define posteriors. Here, five GIN layers without jumping knowledge are applied in which Exponential Linear Unit (ELU) \cite{elu} activation and Batch Normalization \cite{batchnorm} is used on every hidden layer. We use Adam \cite{adam} optimizer with the initial learning rate $0.01$ and decay the learning rate by $0.5$ every $50$ epochs, and set the hidden dimension to be $64$ and batch size to be $128$. For kernel and unsupervised methods, we train the graph representation model for $100$ epochs and use $C$-SVM implementation of LIBSVM \cite{libsvm} to compute the average accuracy and standard deviation across the 10-folds within the cross-validation, where $C$ is set to be $5$. For the supervised method, we also train the classification network for $100$ epochs and report the 10-fold cross-validation accuracy with the standard deviation. %As shown in Table \ref{apptab0}, we list the other crucial parameter values of the proposed method used in this paper. 
To achieve a fair comparison, we use the same seeds to ensure the same validation settings for all models. All codes of neural network based methods in this paper are implemented with PyTorch Geometric\footnote{\url{https://github.com/pyg-team/pytorch_geometric}} and graph kernel methods are implemented with GraKeL\footnote{\url{https://github.com/ysig/GraKeL}}, and our codes are released on this site\footnote{\url{https://github.com/xiaolongo/MVGIB}}. 

\subsection{Graph Classification Results}

Table \ref{tab11} shows the graph classification results in comparison of the proposed method and baselines on the graph benchmark datasets. From this table, we have the following observations. First, compared with the graph kernel methods, the proposed MVGIB exhibits favorable classification accuracy on all benchmark datasets, achieving competitive performance. Second, compared with the supervised methods, the proposed MVGIB consistently obtains better performance. Note that the performance of GIN and KGIN focusing only on a single view is not satisfactory, yet our approach achieves a large performance improvement by using GIN as the base encoder, illustrating the effectiveness of integrating multiviews. Third, compared with unsupervised methods, the proposed MVGIB achieves better results on all datasets with a significant margin. Especially, MVGIB achieves maximum relative improvements of $4.9\%$ on REDDIT-BINARY. Although MVGRL also integrates multiview input graph data, we find that MVGIB achieves better results due to the fact that MVGRL does not distinguish consistency and complementary of multiview latent representations, indicating the necessity of the proposed strategy. In general, the proposed MVGIB outperforms the previous methods consistently and significantly, demonstrating the effectiveness and necessity of the proposed method.  

\subsection{Graph Clustering Results}
To evaluate the cluster performance, we utilize the proposed MVGIB to generate graph representations and then cluster these graph representations using the K-Means algorithm on IMDB and REDDIT datasets. Here, we set the number of clusters to be the ground-truth classes and report the Normalized Mutual Information (NMI) and Adjusted Rand Index (ARI) scores. Table \ref{tab3} shows the graph clustering experimental results. From this table, we have the following two observations. First, compared with the previous DGI and InfoGraph, the multiview baseline, MVGRL, achieves a competitive performance on NMI and ARI scores, indicating the effectiveness of incorporating multiview graph augmentations. Second, we can observe that the proposed MVGIB achieves good performance on NMI and ARI scores compared with the previous methods, further indicating the effectiveness of the proposed method under unsupervised representation learning.

\begin{figure*}[t]
	\centering
	\subfigure[Add edges on PTC-MR]{\includegraphics[width=0.5\columnwidth]{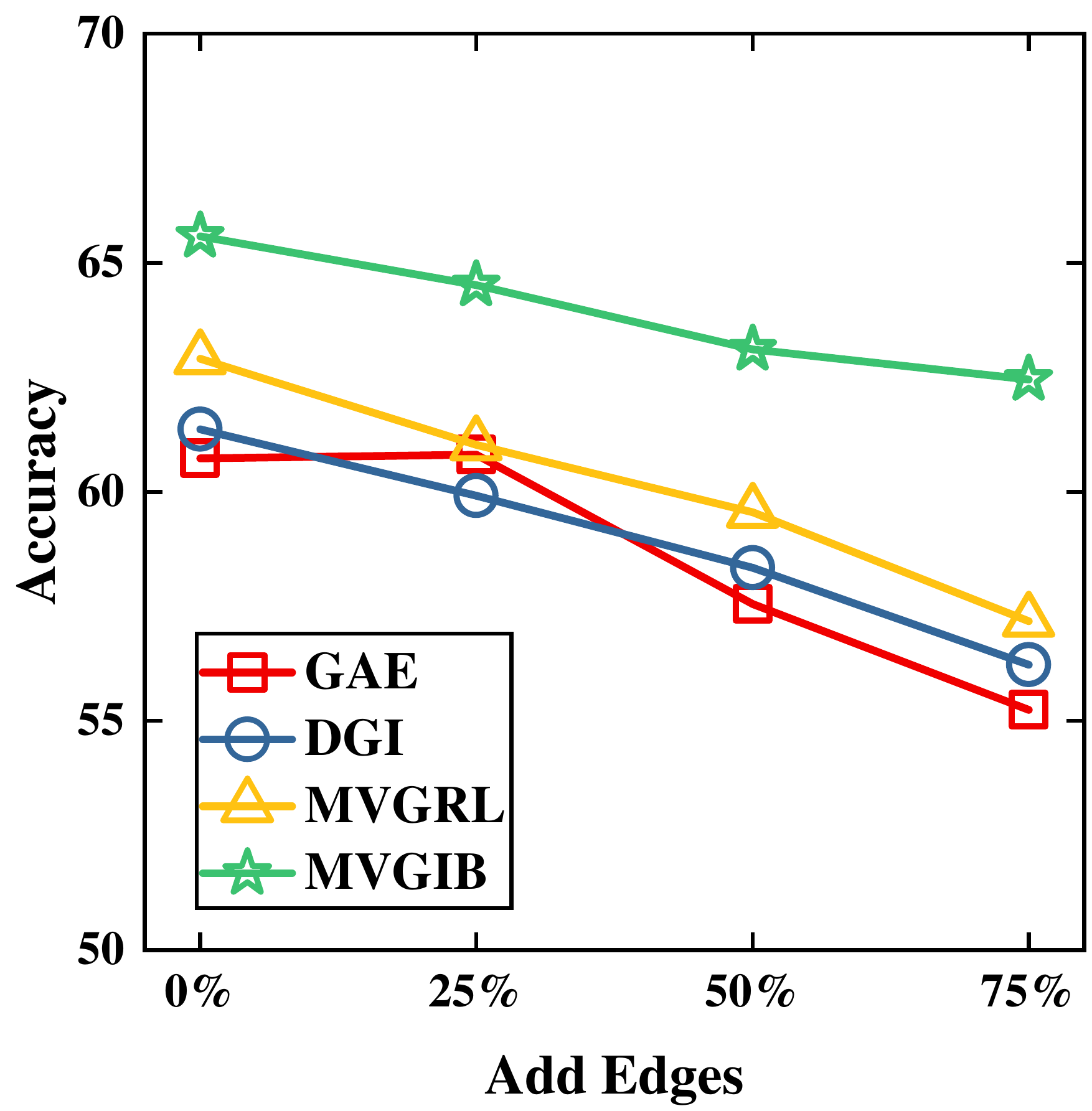}}
	\subfigure[Remove edges on PTC-MR]{\includegraphics[width=0.5\columnwidth]{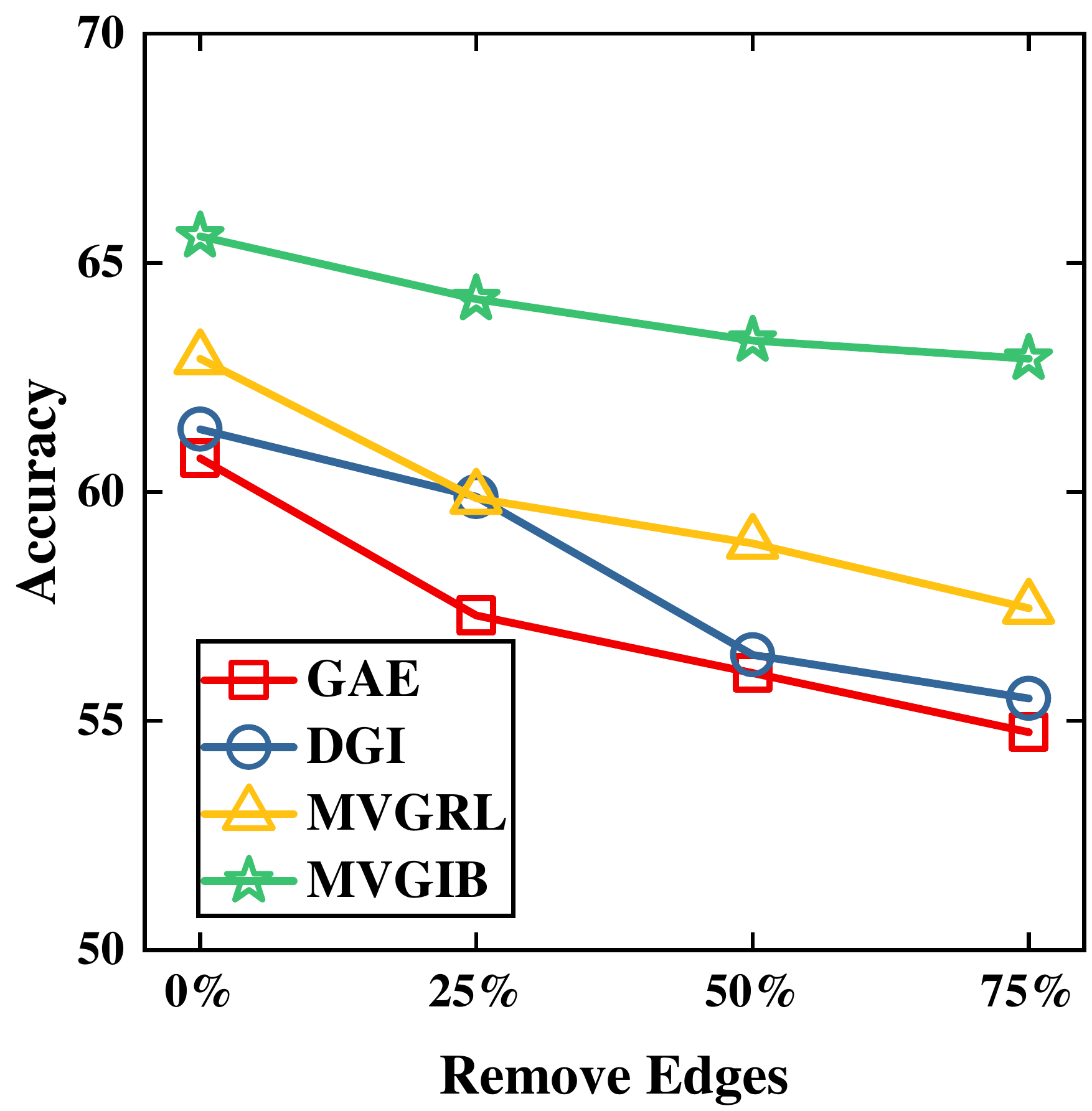}}
	\subfigure[Add edges on MUTAG]{\includegraphics[width=0.5\columnwidth]{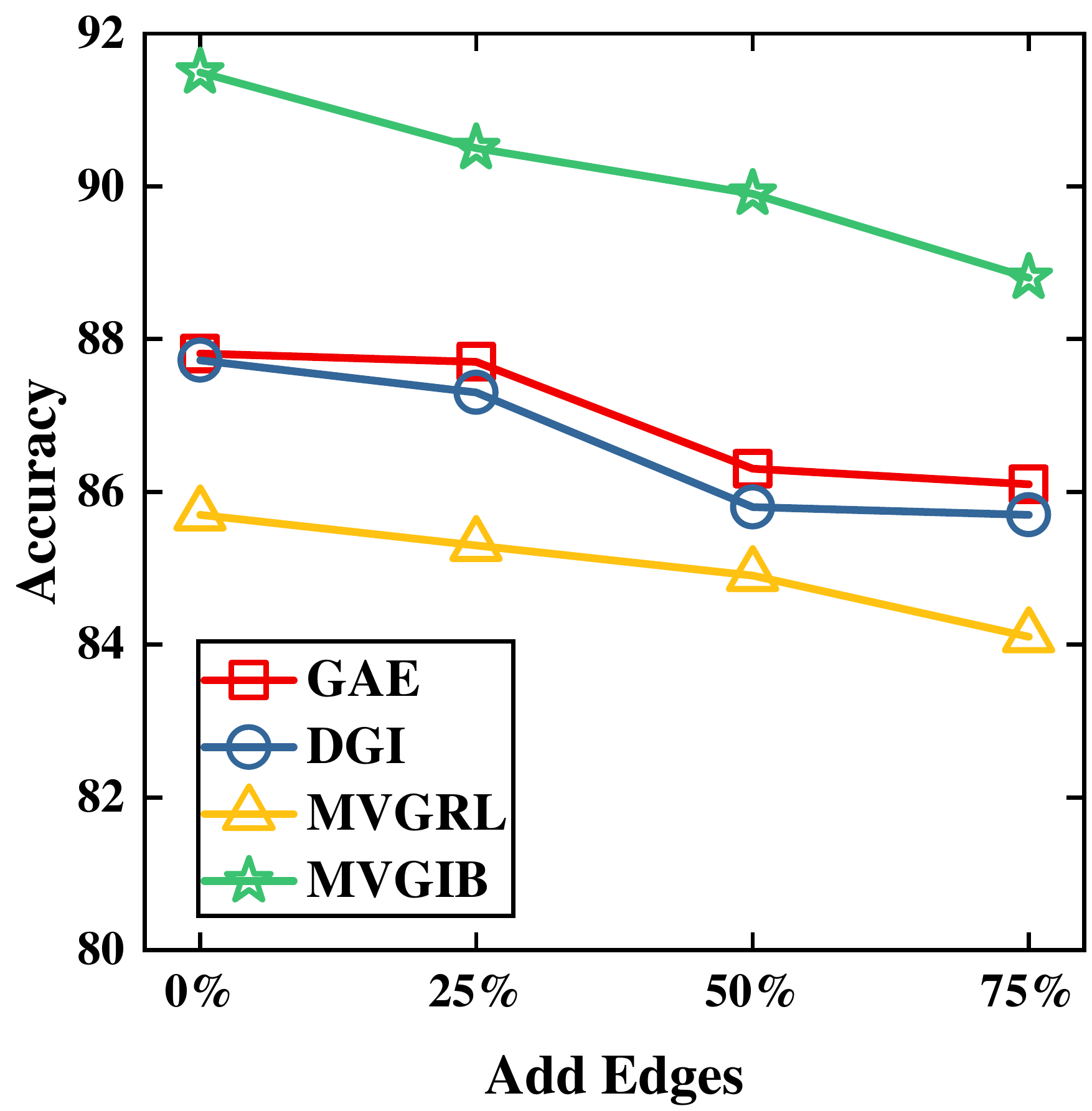}}
	\subfigure[Remove edges on MUTAG]{\includegraphics[width=0.5\columnwidth]{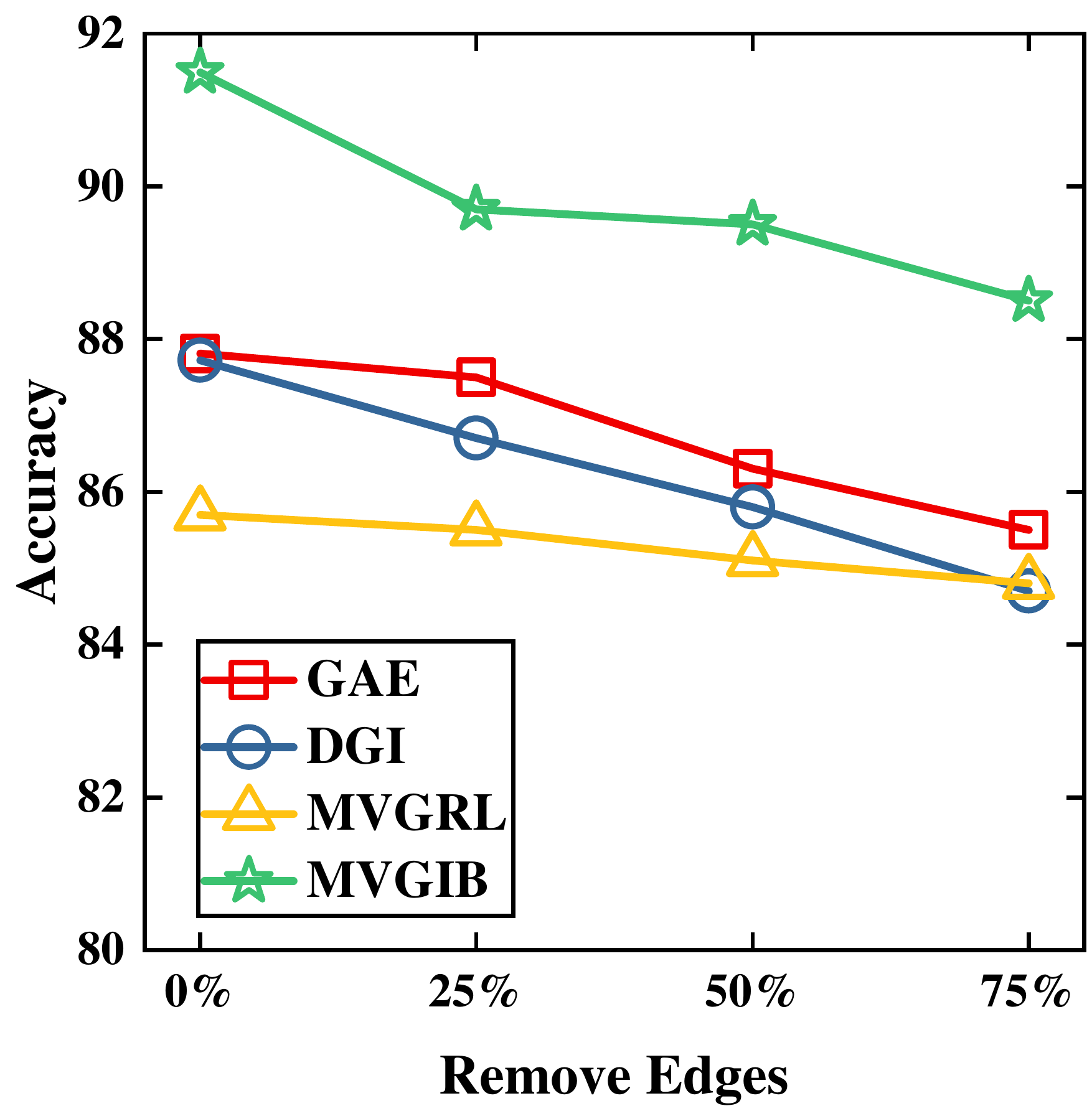}}
	\caption{Cross-validation accuracy under random edge addition and deletion on PTC-MR and MUTAG datasets. The proposed MVGIB outperforms baselines under different edge perturbation rates by a margin and the margin becomes larger as the perturbation rate increases.}
	\label{graph2}
\end{figure*}

\subsection{Ablation Study}
In this subsection, ablation study is conducted to investigate the performance of different components in the proposed method to understand the contribution for multiview graph information bottleneck.

\subsubsection{Impact of Feature Reconstruction}
Since the input graph contains both topology and node features, we not only employ the graph decoders \cite{vgae} to decode the hidden representations to reconstruct the topology but additionally use multilayer perception to reconstruct the node features. To reveal the impact of the feature reconstruction module, we conduct the comparative experiments for all datasets using two variants, i.e., MVGIB w/o feature reconstruction and MVGIB w/ feature reconstruction, and report the 10-fold cross-validation accuracy with the standard deviation. The experimental results are shown in Table \ref{tab22}. From this table, we can observe that with the feature reconstruction generally performs better than without the feature reconstruction, indicating the necessity of feature reconstruction. In particular, for social network datasets (i.e., IMDB and REDDIT datasets), the node features are annotated with node degree, hence the feature reconstruction is actually reconstructing the node degree, which is similar to NWR-GAE \cite{nwrgae} and has been verified to be effective for graph representation learning.

\subsubsection{Impact of Consistency and Complementarity}
Note that the key idea of the proposed method is modeling consistency and complementarity under the multiview setting. To investigate the impact of consistency and complementarity, we establish \mbox{two} variants, i.e., MVGIB w/o complementarity and MVGIB w/o consistency, and report the 10-fold cross-validation accuracy with the standard deviation. The experimental results are shown in Table \ref{tab22}. From this table, several observations can be received. First, compared with MVGIB w/o consistency, MVGIB w/o complementarity achieves better performance. In fact, MVGIB w/o complementarity is similar to MVGRL and can achieve comparable results. Second, despite the poor performance of MVGIB w/o consistency, modeling and integrating consistency can still yield significant performance improvements. 

\subsubsection{Impact of Different Views}
To evaluate the impact of different graph views, we utilize four input graph views including the original Adjacency matrix (\textsc{Adj}), K-Nearest Neighbor matrix (\textsc{Knn}), Dilated K-Nearest Neighbor matrix (\textsc{Dknn}), and Personalized PageRank (\textsc{Ppr}). Here, Dilated K-Nearest Neighbor aims to find the top similar $K-D$ nodes for each node where $D$ is the dilation factor. The 10-fold cross-validation accuracy with the standard deviation are shown in Table \ref{tab22}. From this table, we have the following observations. First, encoding from the \textsc{Ppr} input view and other input views obtains unsatisfactory performance. Second, encoding from \textsc{Adj} and KNN/DKNN input views achieves better performance on all datasets, indicating the effectiveness of integrating feature graph and topology graph for graph representation learning. 

\subsubsection{Impact of Trade-off Parameters}
The overall objective contains four trade-off parameters, i.e., $\delta$, $\eta$, $\zeta$, and $\xi$. For simplicity, we set $\beta = \delta = \eta$ and $\gamma = \zeta = \xi$. Then we vary $\beta$ and $\gamma$ in $\{0.1, 0.3, 0.5, 0.7, 1\}$ and conduct validation experiments on PTC-MR, PROTEINS, IMDB-BINARY, IMDB-MULTI, REDDIT-BINARY, and REDDIT-M5K datasets to study the impact of different trade-off parameters $\beta$ and $\gamma$. The 10-fold cross-validation results are shown in Figure \ref{graph1}. From this figure, we can empirically observe that the proposed MVGIB achieves better performance when $\beta$ and $\gamma$ are close or even equal, and vice versa. Hence, we simply set the trade-off parameters $\beta = \gamma = 1$ in the proposed method.

\subsubsection{Robustness on Edge Perturbation}
To evaluate the robustness of incorporating KNN or DKNN views in the proposed method, we perform the comparative experiments by adding and removing edges on PTC-MR dataset. Here, we randomly add and remove $25\%$, $50\%$, $75\%$ of original edges if edges exist or no such edges, and report the 10-fold cross-validation accuracy in Figure \ref{graph2}. From this figure, we observe that the proposed MVGIB outperforms baselines under different edge perturbation rates by a margin and the margin becomes larger as the perturbation rate increases, indicating the robustness of incorporating KNN or DKNN views. We argue that when the graph is perturbed by noisy edges, the KNN or DKNN graph still contains rich topology information of the original graph. Thus MVGIB can capture the critical information beneficial to the downstream tasks from the KNN or DKNN input view, improving topology robustness.

\section{Conclusion} \label{s6}
In this paper, we propose a novel Multiview Variational Graph Information Bottleneck (MVGIB) principle for graph representation learning. Specifically, to model the consistency across multiviews, we develop an information bottleneck objective that encourages the latent representation contains as much information as possible about the corresponding input view and as less possible about the other view. To model the complementarity, we also develop an information bottleneck objective to retain as much information as possible about the corresponding input view and the other view. We derive the lower and upper bounds of mutual information terms and instantiate the information objective by using graph neural networks. Extensive experiments on seven graph benchmark datasets verify the effectiveness of the proposed method. Currently, integrating more graph views is a challenge and requires further exploration in future work.

\bibliographystyle{IEEEtran}
\bibliography{reference}

\end{document}